\documentclass[nohyperref]{article}

\usepackage{microtype}
\usepackage{graphicx}
\usepackage{subfigure}
\usepackage{booktabs} % for professional tables

\usepackage{hyperref}

\usepackage[accepted]{icml2022}

\usepackage{amsmath}
\usepackage{amssymb}
\usepackage{mathtools}
\usepackage{amsthm}

\usepackage{amsmath}
\usepackage{amssymb}
\usepackage{mathtools}
\usepackage{amsthm}
\usepackage{amsfonts}
\usepackage{bbm}
\usepackage{tikz}           % for drawing
\usetikzlibrary{automata, arrows, arrows.meta, positioning}
\usepackage{mathrsfs}       % for different P's than \mathcal{P}
\usepackage{graphicx}       % for includegraphics
\usepackage{nicefrac}       % for nicer fractions
\usepackage{enumitem}
\usepackage{subfigure}

\renewcommand{\P}{\mathbb{P}}   % probability measure
\newcommand{\trans}{\mathcal{P}}
\newcommand{\E}{\mathbb{E}}     % expectation
\renewcommand{\S}{\mathcal{S}}  % state space
\newcommand{\A}{\mathcal{A}}    % agents
\newcommand{\Reals}{\mathbb{R}} % real numbers
\renewcommand{\r}{r}            % reward function

\newcommand{\influence}{\mathcal{I}}
\newcommand{\truereward}{\r^*}                % true reward function
\newcommand \defn {\mathrel{\triangleq}}
\newcommand{\feasible}{\mathcal{R}} % feasible set of reward functions
\newcommand{\optPi}{\Pi^{\texttt{opt}}}

\newcommand{\loss}{\mathscr{L}} % Loss of a policy
\DeclareMathOperator*{\argmax}{arg\,max}

\DeclareMathOperator*{\spn}{span}

\DeclareMathOperator{\Aff}{Aff}
\newcommand{\acta}{a}
\newcommand{\actb}{b}

\newcommand{\minset}{\Aff} 

\usepackage{apptools}   % for Lemmas in Appendix
\AtAppendix{\counterwithin{lemma}{section}}  % for Lemmas in Appendix
\newtheorem{innercustomgeneric}{\customgenericname}
\providecommand{\customgenericname}{}
\newcommand{\newcustomtheorem}[2]{%
  \newenvironment{#1}[1]
  {%
   \renewcommand\customgenericname{#2}%
   \renewcommand\theinnercustomgeneric{##1}%
   \innercustomgeneric
  }
  {\endinnercustomgeneric}
}

\newcustomtheorem{customthm}{Theorem}
\newcustomtheorem{customlemma}{Lemma}
\newcustomtheorem{customprop}{Proposition}
\newcustomtheorem{customcor}{Corollary}

\theoremstyle{plain}
\newtheorem{remark}{Remark}

\newtheorem{theorem}{Theorem}
\newtheorem{proposition}{Proposition}
\newtheorem{lemma}{Lemma}
\newtheorem{corollary}{Corollary}
\theoremstyle{definition}
\newtheorem{definition}{Definition}

\usepackage[capitalize,noabbrev]{cleveref}

\usepackage[textsize=tiny]{todonotes}

\icmltitlerunning{Interactive Inverse Reinforcement Learning for Cooperative Games}

\begin{document}

\twocolumn[
\icmltitle{Interactive Inverse Reinforcement Learning 
            for Cooperative Games}

% It is OKAY to include author information, even for blind
% submissions: the style file will automatically remove it for you
% unless you've provided the [accepted] option to the icml2022
% package.

% List of affiliations: The first argument should be a (short)
% identifier you will use later to specify author affiliations
% Academic affiliations should list Department, University, City, Region, Country
% Industry affiliations should list Company, City, Region, Country

% You can specify symbols, otherwise they are numbered in order.
% Ideally, you should not use this facility. Affiliations will be numbered
% in order of appearance and this is the preferred way.
% \icmlsetsymbol{equal}{*}

\begin{icmlauthorlist}
\icmlauthor{Thomas Kleine Büning}{yyy}
\icmlauthor{Anne-Marie George}{yyy}
\icmlauthor{Christos Dimitrakakis}{yyy,zzz,xxx}
\end{icmlauthorlist}

\icmlaffiliation{yyy}{Department of Informatics, University of Oslo, Oslo, Norway}
\icmlaffiliation{zzz}{Department of Computer Science, University of Neuchatel, Neuchatel, Switzerland}
\icmlaffiliation{xxx}{Department of Computer Science and Engineering, Chalmers University of Technology, Gothenburg, Sweden}

\icmlcorrespondingauthor{Thomas Kleine Büning}{thomkl@ifi.uio.no}

% You may provide any keywords that you
% find helpful for describing your paper; these are used to populate
% the "keywords" metadata in the PDF but will not be shown in the document
\icmlkeywords{Machine Learning, Reinforcement Learning, Inverse Reinforcement Learning, Preference Learning, Cooperative, Multi-Agent, ICML}

\vskip 0.3in
]

% this must go after the closing bracket ] following \twocolumn[ ...

% This command actually creates the footnote in the first column
% listing the affiliations and the copyright notice.
% The command takes one argument, which is text to display at the start of the footnote.
% The \icmlEqualContribution command is standard text for equal contribution.
% Remove it (just {}) if you do not need this facility.

\printAffiliationsAndNotice{}  % leave blank if no need to mention equal contribution
% \printAffiliationsAndNotice{\icmlEqualContribution} % otherwise use the standard text.

\newif\iflong
\longtrue   % not finished
\longfalse

\begin{abstract}
We study the problem of designing autonomous agents that can learn to cooperate effectively with a potentially suboptimal partner while having no access to the joint reward function. This problem is modeled as a cooperative episodic two-agent Markov decision process. We assume control over only the first of the two agents in a Stackelberg formulation of the game, where the second agent is acting so as to maximise expected utility given the first agent's policy. 
How should the first agent act in order to learn the joint reward function as quickly as possible and so that the joint policy is as close to optimal as possible?
We analyse how knowledge about the reward function can be gained in this interactive two-agent scenario. 
We show that when the learning agent's policies have a significant effect on the transition function, the reward function can be learned efficiently. 
% We show that interacting with an expert implicitly generates observations under different transition dynamics, which allows us to learn the reward function precisely and efficiently. 

% POINT OF THE PAPER IN ONE SENTENCE: 
% interacting is better than just observing. 

% This is in contrast to the standard inverse reinforcement learning setting, where inferring the true reward function is generally not possible.
\end{abstract}

\section{Introduction}
Recent applications of autonomous systems in our daily lives show that autonomous agents are no longer deployed in isolation only, but in situations where they are in close interaction with humans. To facilitate successful and safe cooperation between autonomous systems and humans, we need to design agents that can learn about human preferences as well as adapt to suboptimal human behaviour. We focus on the situation where the autonomous agent and the human simultaneously act in the same environment. As a result, observed human behaviour, which could be used to infer preferences, depends on the learning agent's actions. This leads to the problem of learning preferences and intentions from interactions. Learning in these interactive scenarios brings its own challenges, but also significant benefits as we will see in the following.

% We are in particular interested in situations where the autonomous agent and the human jointly act in the same environment. As a result, the demonstrations from the human, which could be used to learn their reward function, depend on the autonomous agent's actions. We are thus interested in constructing agents that can learn their (human) partner's intentions from interactions. Learning about preferences in these interactive scenarios brings its own challenges, but also significant benefits as we will see in the following. 

% However, when autonomous agents and humans jointly act in the same environment, observed human preferences depend on the learning agent's actions. Thus, instead of learning from observations such agents must be able to learn from interactions with their human partner. Learning about preferences in these interactive scenarios brings its own challenges, but also significant benefits as we will see in the following. 

% To this end, it is interesting to design autonomous agents that are capable of learning their (human) partner's preferences from interactions. 
% Learning about preferences in these interactive scenarios brings its own challenges, but also significant benefits as we will see in the following. 

In this paper, we consider the problem of learning to cooperate with a potentially suboptimal partner while having no access to the joint reward function. 
This problem is modeled as a {cooperative} episodic Markov Decision Process (MDP) between two agents $\A_1$ and $\A_2$.
While agent $\A_2$ (the human) knows the joint reward function, we take the perspective of agent $\A_1$ (the learner) that has to cooperate with $\A_2$ without knowing or observing the rewards. 
As an example, consider a maze in which the human tries to reach a target while the learning agent can unlock doors to help the human move, but without knowing the precise target location.
We focus on the Stackelberg formulation of the game, in which at the beginning of each episode the learner commits to a policy before the human does. 
This allows us to view the learning agent as a \emph{designer of environments} that the human operates in. For instance, when the learning agent's actions correspond to unlocking doors in a grid world, then, in the Stackelberg game, we can interpret the learner's policy as choosing a maze layout, which is communicated to the human at the beginning of the episode and in which she operates.

Inverse Reinforcement Learning (IRL) \citep{russell1998learning} can be used to infer the reward function of an agent from observations of that agent's behaviour, which is assumed to be (near-)optimal. In our case, the learner also obtains observations of the human's behaviour through interactions, which could then be used to infer the joint reward function. However, 
% By observing the human and assuming their behaviour to be (approximately) optimal, we can then learn about the joint reward function using Inverse Reinforcement Learning (IRL). 
the human's actions, e.g.\ the path taken in a maze, depend on the learner's policy, e.g.\ the maze layout, so that in contrast to the standard IRL formulation the learner now \emph{actively influences} the demonstrations of the human expert. 
This leads to an interesting Interactive IRL setting, where the learner can actively seek information about the joint reward function by playing specific policies. 
% In combination with the role of the learner as a designer of environments, this leads to an interesting interactive IRL scenario. 
In this paper, we analyse how to infer the unknown (joint) reward function from interactions with the expert and how the learner should choose its policy so that the two agents collaborate efficiently over both the short and long term. 
We lay an emphasis on the role of the learner as the designer of environments and investigate what environments allow the learning agent to infer the reward function quickly while achieving high levels of cooperation. 
% We emphasise the aspect of designing environments for the human, which allow us to learn the reward function quickly. 

% \textcolor{purple}{Call $\A_1$ the "helper agent", "learner", "learning agent", "AI agent"; \\ Call $\A_2$ the "human", the "expert",  the "second agent".} 

% FROM A REVIEW: Relatedly, I think the introduction should motivate when this sort of formulation would be useful in the real world in cooperative settings. Could you please add a paragraph that discusses practical implications of the proposed formalism, example settings where it could be useful, and why existing formulations are insufficient for these examples? Similarly, as the setting in Sec. 3 is introduced, it would be useful to clarify why certain assumptions are reasonable for the problem formulation that is being considered (e.g. when is it reasonable to assume that A2 will know A1’s policy ahead of time? %If A1 can communicate their policy to A2, why can’t A2 communicate the reward to A1?).

\paragraph{Outline and Contribution.}
We discuss related work in Section~\ref{sec:related-work} and formally introduce the setting in Section~\ref{section:setting}. 
Section~\ref{section:cooperating_with_optimal} considers the case where $\A_2$ plays \emph{optimally}. 
We show how to learn about the reward function from interactions with $\A_2$ and prove the existence of ideal reward learning environments.
We then construct an algorithm that is no-regret under mild assumptions. % in Section~\ref{subsection: algo-for-opt-responses}.
Section~\ref{sec:suboptimal} considers the case where $\A_2$ responds \emph{suboptimally}.
In Section~\ref{subsection:bayesian_IRL}, we adapt conventional Bayesian IRL methods for estimating the reward function to our setting. We then analyse optimal commitment strategies for cooperating with suboptimal followers in Section~\ref{subsection:planning_suboptimal}.
Section~\ref{sec:experiments} describes the experiments, which we perform on random MDPs and specially constructed maze problems. 
% We compare our algorithm with the closest equivalent inverse reinforcement learning algorithms. 
Our experiments support our theoretical results and show that the interactive nature of our setting allows the learning agent to obtain a much better estimate of the reward function (compared to the standard IRL setting). We thus achieve better cooperation by intelligently probing the human's responses. Future work is discussed in Section~\ref{section:future_work}. Finally, omitted proofs, experimental details and algorithms are collected in the Appendix. 
\section{Related Work} 
\label{sec:related-work}
Since our setting requires
% Our setting concerns two problems: 
(a) inferring the joint reward function, as in IRL, and (b) collaborating with a potentially suboptimal agent, in this section we present related work in those two domains.

\paragraph{Inverse Reinforcement Learning.} 
IRL \citep{russell1998learning} aims to find a reward function that explains observed behaviour of an agent. We face the same problem, with the main difference being that \emph{two} agents act in the environment simultaneously, one of which (the human) knows the reward function and the other (the learner) does not. 
% This is in contrast to other work on IRL. 
Our algorithm for the case when $\A_2$ is optimal is based on a characterisation of reward functions consistent with an optimal policy, similarly to~\citet{ng2000IRL}.
%This exposes the ill-posed nature of the IRL problem as there always exists an infinite number of reward functions explaining an optimal policy.
We extend their characterisation to our interactive setting and prove the existence of ideal (reward) learning environments. 
\citet{ramachandran2007BIRL} adopt a Bayesian perspective to the IRL problem as it provides a principled way to reason under uncertainty. The Bayesian formulation of the IRL problem can naturally account for suboptimal demonstrations as well as partial information and we will show how to translate the Bayesian approach to our interactive IRL setting. 
% Additional Bayesian IRL citations that we could include are, e.g., \cite{choi2011MAP-BIRL, rothkopf2011preference, michini2012bayesian, park2020inferring, chan2021scalable}

\citet{hadfield2016cooperative} introduce the problem of cooperative IRL in which a robot must cooperate with a human but does not initially know the reward function. 
Their work focuses on apprenticeship learning, where the robot and the human \emph{take turns} demonstrating and performing a task. In particular, they examine the problem of calculating optimal human demonstrations for the robot to observe.
% when the AI is known to perform IRL. 
Instead, we consider the situation when the agents \emph{interact} by simultaneously acting in the same environment.
Our setting also notably differs from apprenticeship learning \citep{abbeel2004apprenticeship} and imitation learning \citep{ratliff2006maximum} more generally in that our goal is {\em not} to mimic the behaviour of $\A_2$, as effective cooperation between $\A_1$ and $\A_2$ may require both agents to perform entirely different tasks. % and action spaces do not have to coincide. 
\citet{nikolaidis2013human} consider a cross-training approach in which a human expert and a robot repeatedly switch roles. In the first of two phases, the expert operates in an environment, which is influenced by the robot. The learner then observes the expert and updates its estimates of the reward function. In the second phase, the robot then demonstrates the learned policy while the expert influences the transitions. Crucially, in this approach the human steers the learning of the robot similar to teaching approaches for IRL \citep{brown2019machine, parameswaran2019interactive}. In contrast, we consider the situation where the learner actively seeks information from the human over whom we have no control. 
% A reviewer strongly suggested to include the Nikolaidis and Shah paper. 
\citet{natarajan2010multi} consider a multi-agent extension of IRL in which the learner observes multiple experts maximising a joint reward function. Similarly, \citet{lin2019multi} address the problem of multi-agent IRL in certain general-sum games. %where a learner observes multiple agents cooperating or competing. 
In contrast to their work, we consider the case where the learner is {\em not} a passive observer, but interacts with the other agent and thereby influences what observations it collects. 
%

% Also related is ..., as they consider the problem of designing environments to promote certain behaviour. 
\citet{zhang2008enabling} and \citet{zhang2009policy} consider the problem of \emph{environment design}: how to modify an environment so as to influence an agent's decisions. 
They analyse how to construct \emph{reward incentives} to induce a particular policy when the reward function of the acting agent is unknown. In our setting, we can also view the learner as a designer of environments that the human operates in, however, with the difference that the learner influences transitions, but not the underlying reward function. Moreover, our goal is generally not to steer the human towards certain behaviour, but rather to learn from and cooperate with a human expert.

\paragraph{Cooperating with suboptimal partners.}
%\paragraph{Human-AI Collaboration.}
In the context of human-AI collaboration, there have been recent efforts addressing the problem of cooperating with a potentially suboptimal partner \emph{when the reward function is known.}
In particular, \citet{dimitrakakis2017multi} and \citet{radanovic2019learning} consider a setting where the human responds suboptimally to the learning agent's policy. The former focuses on a single-stage Stackelberg game, while the latter on an online learning variant of the problem. However, in both cases the learning agent knows the human's reward function. 
% The main aim of this paper is to remove this unrealistic assumption. 

Our work also has some links to the problem of optimal commitment in Stackelberg games \citep{conitzer2006computing, letchford2012computing}. While prior work assumes optimal responses and a potentially competitive game, we focus on finding optimal commitment strategies when playing with a {\em suboptimal} follower in a strictly {\em cooperative} setting. % In that setting, we show that a dominating strategy may not exist. 

\section{Setting}\label{section:setting}
We model this problem as a cooperative two-agent MDP $(\S, A_1, A_2, \trans, \r, \gamma)$ between agents $\A_1$ and $\A_2$, where $\S$ denotes a finite state space, $A_i$ a finite action space of agent $\A_i$ with $i \in \{1,2\}$, 
$\trans: \S \times A_1 \times A_2 \to \Delta(\S)$ the transition function,
%where $\Delta(\S)$ denotes the unit simplex, 
$\r: \S \to \Reals$ the \emph{joint} reward function and $\gamma \in [0, 1)$ the discount factor. We will take the perspective of agent $\A_1$ that, without knowing or observing the joint reward function, aims to cooperate with its partner $\A_2$. We assume that the interaction between the two agents and the environment takes place in a sequence of episodes, where at the beginning of each episode, $\A_1$ commits to a policy $\pi^1$ first. Agent $\A_2$ then responds with a policy $\pi^2$ and the joint policy is executed until the end of the episode.\footnote{Even in MDPs without termination condition, discounting corresponds to episodes that end with probability $1-\gamma$ each time step.} %We briefly discuss undiscounted fixed-horizon episodic MDPs in Appendix~\ref{appendix:finite_horizon_MDPs}.} 
We assume that agents $\A_1$ and $\A_2$ know the transition function.

\paragraph{Interaction.} 
The repeated interaction of both agents can be specified as the following \emph{Stackelberg game}. In episode~$t$:
\begin{itemize}[topsep=0pt, itemsep=0pt]
    \item[1)] $\A_1$ commits to policy $\pi^1_t$,
    \item[2)] $\A_2$ observes $\pi^1_t$ and responds with policy $\pi^2_t$,
    \item[3a)] $\A_1$ observes the \emph{fully specified policy} $\pi^2_t$, \emph{or}
    \item[3b)] $\A_1$ observes a \emph{trajectory} $\tau_t$ of (random) length $H+1$, where $\tau_t = (s_0, \acta_0, \actb_0, \dots,  s_H, \acta_H, \actb_H)$. 
    %$\A_1$ observes a \emph{trajectory} $\tau_t = (s_0, \acta_0, \actb_0, \dots,  s_H, \acta_H, \actb_H)$ of (random) length $H+1$.
\end{itemize}
%This simplifies our analysis as we do not have to prescribe any specific belief structure to $\A_2$\footnote{Possible extensions to simultaneous commitment games are addressed in Section \ref{section:future_work} as future work.}\mgcomment{make sure this is actually discusses in the conclusion} 
Alternative 3a) describes the {\em full information} setting in which the complete policy $\pi^2_t$ is available to the learner at the end of each episode. This could, for instance, be the case when interaction takes place for a sufficiently long time in each episode, or the same policy is committed by $\A_1$ several times so that $\A_1$ can effectively observe $\A_2$'s response.  Alternative 3b) corresponds to the {\em partial information} setting, where $\A_1$ interacts with $\A_2$ in a series of $H+1$ time steps and observes the generated trajectory only.

\subsection{Preliminaries}
By a slight abuse of notation, we sometimes refer to functions $\smash{f : \S \to \Reals}$ as vectors $\smash{f \in \Reals^{|\S|}}$. For instance, when convenient, we treat reward functions ${\r : \S \to \Reals}$ as vectors ${\r \in \Reals^{|\S|}}$. 
Let $V_{\pi^1, \pi^2}$ denote the value function under the joint policy $(\pi^1, \pi^2)$. The value function satisfies the Bellman equation, which we can concisely express in matrix-form as 
\begin{align*}%\label{equation:value_function}
     V_{\pi^1, \pi^2} = (I - \gamma \trans_{\pi^1, \pi^2})^{-1} \r,
\end{align*}
% $\smash{V_{\pi^1, \pi^2} = (I - \gamma \trans_{\pi^1, \pi^2})^{-1} \r}$,
% where $\smash{V_{\pi^1, \pi^2}}$ and $\r$ are column vectors and $\trans_{\pi^1, \pi^2}$ is the transition matrix obtained from $\trans$ by marginalising over policy $\smash{(\pi^1, \pi^2)}$. Accordingly, the $Q$-values under the joint policy $\smash{(\pi^1, \pi^2)}$ are defined as $Q_{\pi^1, \pi^2}(s, \acta, \actb) = \r(s) + \gamma \sum_{s'} \trans(s' | s, \acta, \actb) V_{\pi^1, \pi^2} (s')$ for $s \in \S$, $\acta \in A_1$, $\actb \in A_2$. 
% When $\A_1$ commits to a policy ${\pi^1}$ first, $\A_2$ gets to plan under the marginalised transition function $\trans_{\pi^1} : \S \times A_2 \to \Delta (\S)$ given by $\trans_{\pi^1}(s' | s, \actb) = \E_{a \sim \pi^1} [ \trans(s' | s, \acta, \actb)]$.  Consequently, the $Q$-values for $\A_2$ under transition function $\trans_{\pi^1}$ are defined as $Q_{ \pi^1, \pi^2} (s, \actb) = \E_{\acta\sim \pi^1} [ Q_{\pi^1, \pi^2} (s, \acta, \actb)]$. 
% We denote the corresponding optimal $Q$-values under transition matrix $\trans_{\pi^1}$ by  $Q^*_{\pi^1}(s, \actb) = \max_{\pi^2} Q_{\pi^1, \pi^2} (s, \actb)$. %for $s \in \S$ and $\actb \in A_2$. 

where $\smash{V_{\pi^1, \pi^2}}$ and $\r$ are column vectors and $\trans_{\pi^1, \pi^2}$ is the transition matrix obtained from $\trans$ by marginalising over policy $\smash{(\pi^1, \pi^2)}$. Let $Q_{\pi^1, \pi^2} (s, a, b)$ denote the value of taking joint action $(a,b)$ in state $s$ under policy $(\pi^1, \pi^2)$. 
When $\A_1$ commits to a policy ${\pi^1}$ first, agent $\A_2$ gets to plan under the marginalised transitions $\trans_{\pi^1} : \S \times A_2 \to \Delta (\S)$ given by $\trans_{\pi^1}(s' | s, \actb) = \E_{a \sim \pi^1} [ \trans(s' | s, \acta, \actb)]$. The $Q$-values for $\A_2$ under $\trans_{\pi^1}$ equal $Q_{ \pi^1, \pi^2} (s, \actb) = \E_{\acta\sim \pi^1} [ Q_{\pi^1, \pi^2} (s, \acta, \actb)]$ and we denote the optimal $Q$-value with respect to $\trans_{\pi^1}$ by $Q^*_{\pi^1}(s, \actb) = \max_{\pi^2} Q_{\pi^1, \pi^2} (s, \actb)$.

\paragraph{Behavioural Models for $\A_2$.}
A typical assumption about the behaviour of a partner (or opponent) in game theory \citep{roughgarden_2007} and IRL \citep{ng2000IRL} is that of {optimal} behaviour, sometimes referred to as fully rational behaviour. In our case, this means that in episode $t$, agent $\A_2$ plays an optimal response $\pi^2_t(\pi^1_t)$ to the policy $\pi^1_t$ committed by agent $\A_1$. Note that we will simply write $\pi^2_t$ when the dependence on $\pi^1_t$ is clear from the context. 
%To emphasise that the response of $\A_2$ depends on the commitment of $\A_1$, we then sometimes write $\pi^2(\pi^1)$. % to emphasise that the response of $\A_2$ depends on the commitment of $\A_1$. 

We are also interested in the case when $\A_2$ is {suboptimal}. 
A common decision-model for suboptimal human behaviour in IRL \citep{jeon2020reward}, economics \citep{luce1959individual}, and cognitive science \citep{baker2009action} are Boltzmann-rational policies for which the probability of choosing an action is exponentially dependent on its expected value: 
% A lot of literature on Boltzmann policies is in \cite{jeon2020reward. 
\begin{align*}%\label{eq:boltzmann-policies}
    \pi^2 (\actb\mid s, \pi^1) \propto \exp \big(\beta Q_{\pi^1}^* (s, \actb) \big).
\end{align*}
Here, $\beta \geq 0$ is called the inverse temperature of the distribution and indicates how rationally $\A_2$ is behaving. In particular, for $\beta = 0$, $\A_2$ acts uniformly at random, and for $\beta \to \infty$, $\A_2$ acts perfectly rational, i.e.\ optimally in response to $\A_1$'s committed policy. %\footnote{In some scenarios, other behavioural models may be better suited such as $\varepsilon$-greedy policies.} % For instance, $\varepsilon$-greedy policies
%, which choose an optimal action with probability $1-\varepsilon$ and a uniformly random action with probability $\varepsilon$ at any time step, 
% could be a better model when interacting with a learning artificial agent.}

\paragraph{Objective and Regret.} 
Agent $\A_1$ aims to maximise the expected sum of discounted rewards by learning about the joint reward function and cooperating with $\A_2$. 
In general, due to the possibly suboptimal nature of $\A_2$, we have that $\max_{\pi^1} V_{\pi^1, \pi^2(\pi^1)} \preceq \max_{\pi^1, \pi^2} V_{\pi^1, \pi^2}$, i.e.\ the value of the game under $\A_2$'s behavioural model is bounded by the value of the joint optimal policy.
For an initial state distribution~$D$, % and initial state $s_0 \sim D$,
we define the value of the optimal commitment strategy as  
$$V^* = \max_{\pi^1} \E_{s_0 \sim D} \big[V_{\pi^1, \pi^2(\pi^1)} (s_0)\big],$$
where $\pi^2(\pi^1)$ denotes the response of $\A_2$ to policy $\pi^1$. % of $\A_1$.
%the commitment of $\pi^1$ by agent $\A_1$.
Note that the optimal value $V^*$ may only be well-defined with respect to a specific initial state distribution as a dominating commitment strategy may fail to exist when $\A_2$ responds suboptimally (see Section~\ref{subsection:planning_suboptimal}). 
We define the (per-episode) regret of playing policy $\pi^1$ as the difference
$\loss (\pi^1) = V^* - \E_{s_0\sim D} [V_{\pi^1, \pi^2(\pi^1)} (s_0)]$.
Similarly, we define the (online) regret of playing policies ${\pi^1_1, \dots, \pi^1_T}$ as the sum
$\loss (\pi^1_1, \dots, \pi^1_T) = \sum_{t=1}^T \loss (\pi^1_t)$.

\subsection{Interactive IRL}\label{section:online_IRL}

In the classical IRL problem, the learner is able to observe an expert performing a task. The observations are then interpreted as demonstrations of approximately optimal behaviour in a \emph{fixed} single-agent MDP with unknown reward function. Our setting is substantially different, as two agents must collaborate in the same two-agent MDP, with the first agent not knowing the common reward function. As a result, the second agent's demonstrations depend on the first agent's policy and so become {\em context-dependent}. In addition, learning must take place in an {\em online} fashion, as the first agent must adapt its policy to extract information and to better collaborate. 

\paragraph{$\A_1$ as an MDP Designer.} 
When the learner, $\A_1$, commits to a policy $\smash{\pi^1}$ at the beginning of an episode, then\,---\,with knowledge of $\pi^1$\,---\,the expert, $\A_2$, can be seen as planning in a single-agent MDP with transition function~$\smash{\trans_{\pi^1}}$. Consequently, from the perspective of the learner, choosing a policy $\pi^1$ is equivalent to designing single-agent MDPs for the human expert to act in. While the state space, $\A_2$'s action space, the (unknown) reward function as well as the discount factor remain the same across these simplified MDPs, $\A_2$ may face different environment dynamics~$\trans_{\pi^1}$ depending on $\A_1$'s policy. This is in contrast to the standard IRL setting in which demonstrations always take place in the same fixed MDP. An abstract example where the learner creates different environments for the expert to operate in is illustrated in \cref{figure:introductory_example}\textcolor{mydarkblue}{(a)}.

\paragraph{Context-Dependent Responses.}
%Since the learning agent $\A_1$ can be seen as designing an environment that the expert $\A_2$ acts in, we can 
The learner can now interpret the expert's response to a policy $\pi^1$ as a demonstration in the single-agent MDP $(\S, A_2, \trans_{\pi^1}, \truereward, \gamma)$, where $\truereward$ is the true reward function that is unknown and unobserved by $\A_1$. 
Since $\A_2$ faces possibly different environment dynamics across episodes, we can also expect $\A_2$'s behaviour to vary between episodes. In Figure~\ref{figure:introductory_example}\textcolor{mydarkblue}{(a)}, for instance, the expert adapts their policy 
%takes a different route to the $+1$ reward state 
to the specific maze layout %, i.e.\ the environment, 
created by the learner. As a result, $\A_2$'s responses (and thus demonstrations) become context-dependent in the sense that they always depend on $\A_1$'s policy, i.e.\ the environment that is implicitly generated by $\A_1$.  

In particular, we see that even though the underlying reward function remains the same, the results of IRL methods vary depending on the environment in which demonstrations were provided.
Figure~\ref{figure:introductory_example}\textcolor{mydarkblue}{(b)} also illustrates that reward learning may overfit to specific environment dynamics, which has also been observed by, e.g., \citet{toyer2020Magical}.
While there may exist certain environment dynamics that are better suited for learning rewards, in this paper we focus on designing a sequence of environments, based on past data, to learn the reward function efficiently.
% We will analyse the effect of the transition dynamics on our ability to learn the reward function and focus on designing a sequence of environments, based on past data, to learn the reward function efficiently.

\begin{figure}[t]
    \vspace{0.1cm}
    \colorlet{DarkGreen}{green!25!black!75}
\colorlet{DarkRed}{red!100}

\begin{tikzpicture}
    \node at (-.51, 0.75) {\small (a)};

    \fill[black!65] (0.5, 0.5) -- +(0, .5) -- +(.5, .5) -- +(.5,0) -- cycle;
    \fill[black!65] (1.0, 0.5) -- +(0, .5) -- +(.5, .5) -- +(.5,0) -- cycle;
    \fill[black!65] (1.5, 0.5) -- +(0, .5) -- +(.5, .5) -- +(.5,0) -- cycle;
    \draw[step=0.5cm,color=black!30] (0,0) grid (2.0,1.5);
    \node at (0.25, 0.75) {\small $\A_2$};
    
    \node at (1.75, 1.25) {\small \textcolor{DarkGreen}{{+1}}};
    \draw[-, gray!99] (0.25, 0.95) -- (0.25, 1.25);
    \draw[-to, gray!99] (0.25, 1.25) -- (1.5, 1.25);
\end{tikzpicture}
\hspace{0.1cm}
% top mid
\begin{tikzpicture}
    \fill[black!65] (1, 0) -- +(0, .5) -- +(.5, .5) -- +(.5,0) -- cycle;
    \fill[black!65] (1, 1) -- +(0, .5) -- +(.5, .5) -- +(.5,0) -- cycle;
    \fill[black!65] (0.5, 1) -- +(0, .5) -- +(.5, .5) -- +(.5,0) -- cycle;
    \draw[step=0.5cm,color=black!30] (0,0) grid (2.0,1.5);
    \node at (0.25, 0.75) {\small $\A_2$};

    \node at (1.75, 1.25) {\small \textcolor{DarkGreen}{{+1}}};
    \draw[-, gray!99] (0.45, 0.75) -- (1.75, 0.75);
    \draw[-to, gray!99] ((1.75, 0.75) -- (1.75, 1);

\end{tikzpicture}
\hspace{0.1cm}
% top right
\begin{tikzpicture}
    \node at (0.25, 0.75) {\small $\A_2$};
    \fill[black!65] (0.5, 0.5) -- +(0, .5) -- +(.5, .5) -- +(.5,0) -- cycle;
    \fill[black!65] (1.0, 0.5) -- +(0, .5) -- +(.5, .5) -- +(.5,0) -- cycle;
    \fill[black!65] (1.0, 1.0) -- +(0, .5) -- +(.5, .5) -- +(.5,0) -- cycle;
    \draw[step=0.5cm,color=black!30] (0,0) grid (2.0,1.5);
    
    \node at (1.75, 1.25) {\small \textcolor{DarkGreen}{{+1}}};
    \draw[-, gray!99] (0.25, 0.55) -- (0.25, 0.25);
    \draw[-, gray!99] (0.25, 0.25) -- (1.75, 0.25);
    \draw[-to, gray!99] (1.75, 0.25) -- (1.75, 1);
\end{tikzpicture} 
    \vspace{0.2cm}
    \colorlet{DarkGreen}{green!25!black!75}
\colorlet{DarkRed}{red!100}

\begin{tikzpicture}      
    \node at (-.5, 0.75) {\small (b)};

    % Fill with Colour depending on Reward Function estimated by BIRL
    % 1st row
    \fill[DarkGreen!0.2] (0, 1) -- +(0, .5) -- +(.5, .5) -- +(.5,0) -- cycle;
    \fill[DarkGreen!10.8] (.5, 1) -- +(0, .5) -- +(.5, .5) -- +(.5,0) -- cycle;
    \fill[DarkGreen!27.3] (1, 1) -- +(0, .5) -- +(.5, .5) -- +(.5,0) -- cycle;
    \fill[DarkGreen!70] (1.5, 1) -- +(0, .5) -- +(.5, .5) -- +(.5,0) -- cycle;
    % 2nd row
    \fill[DarkGreen!0] (0, .5) -- +(0, .5) -- +(.5, .5) -- +(.5,0) -- cycle;
    \fill[DarkGreen!5.62] (.5, .5) -- +(0, .5) -- +(.5, .5) -- +(.5,0) -- cycle;
    \fill[DarkGreen!4.5] (1, .5) -- +(0, .5) -- +(.5, .5) -- +(.5,0) -- cycle;
    \fill[DarkGreen!5.7] (1.5, .5) -- +(0, .5) -- +(.5, .5) -- +(.5,0) -- cycle;
     % 3rd row
    \fill[DarkGreen!0.77] (0, 0) -- +(0, .5) -- +(.5, .5) -- +(.5,0) -- cycle;
    \fill[DarkGreen!0] (.5, 0) -- +(0, .5) -- +(.5, .5) -- +(.5,0) -- cycle;
    \fill[DarkGreen!0.47] (1, 0) -- +(0, .5) -- +(.5, .5) -- +(.5,0) -- cycle;
    \fill[DarkGreen!1.7] (1.5, 0) -- +(0, .5) -- +(.5, .5) -- +(.5,0) -- cycle;

    \draw[step=0.5cm,color=black!30] (0,0) grid (2.0,1.5);
    \node at (0.25, 0.75) {\small $\A_2$};
    
    % \fill[black!70] (0.5, 0.5) -- +(0, .5) -- +(.5, .5) -- +(.5,0) -- cycle;
    % \fill[black!70] (1.0, 0.5) -- +(0, .5) -- +(.5, .5) -- +(.5,0) -- cycle;
    % \fill[black!70] (1.5, 0.5) -- +(0, .5) -- +(.5, .5) -- +(.5,0) -- cycle;
    
    % % mark blocked cells
    % \draw[-, thick, black!70] (0.5, 0.5) -- (1, 1);
    % \draw[-, thick, black!70] (0.5, 1) -- (1, 0.5);
    
    % \draw[-, thick, black!70] (1, 0.5) -- (1.5, 1);
    % \draw[-, thick, black!70] (1, 1) -- (1.5, 0.5);

    % \draw[-, thick, black!70] (1.5, 0.5) -- (2, 1);
    % \draw[-, thick, black!70] (1.5, 1) -- (2, 0.5);
    
\end{tikzpicture}
\hspace{0.13cm}
\begin{tikzpicture}
    
    % Fill with Colour depending on Reward Function estimated by BIRL
    % 1st row
    \fill[DarkGreen!0] (0, 1) -- +(0, .5) -- +(.5, .5) -- +(.5,0) -- cycle;
    \fill[DarkGreen!7.2] (.5, 1) -- +(0, .5) -- +(.5, .5) -- +(.5,0) -- cycle;
    \fill[DarkGreen!11.5] (1, 1) -- +(0, .5) -- +(.5, .5) -- +(.5,0) -- cycle;
    \fill[DarkGreen!65] (1.5, 1) -- +(0, .5) -- +(.5, .5) -- +(.5,0) -- cycle;
    % 2nd row
    \fill[DarkGreen!7] (0, .5) -- +(0, .5) -- +(.5, .5) -- +(.5,0) -- cycle;
    \fill[DarkGreen!15.8] (.5, .5) -- +(0, .5) -- +(.5, .5) -- +(.5,0) -- cycle;
    \fill[DarkGreen!22.4] (1, .5) -- +(0, .5) -- +(.5, .5) -- +(.5,0) -- cycle;
    \fill[DarkGreen!25.6] (1.5, .5) -- +(0, .5) -- +(.5, .5) -- +(.5,0) -- cycle;
     % 3rd row
    \fill[DarkGreen!3.5] (0, 0) -- +(0, .5) -- +(.5, .5) -- +(.5,0) -- cycle;
    \fill[DarkGreen!8.3] (.5, 0) -- +(0, .5) -- +(.5, .5) -- +(.5,0) -- cycle;
    \fill[DarkGreen!9.23] (1, 0) -- +(0, .5) -- +(.5, .5) -- +(.5,0) -- cycle;
    \fill[DarkGreen!18.5] (1.5, 0) -- +(0, .5) -- +(.5, .5) -- +(.5,0) -- cycle;

    \draw[step=0.5cm,color=black!30] (0,0) grid (2.0,1.5);
    \node at (0.25, 0.75) {\small $\A_2$};
\end{tikzpicture}
\hspace{0.13cm}
\begin{tikzpicture}
    
    % Fill with Colour depending on Reward Function estimated by BIRL
    % 1st row
    \fill[DarkGreen!1.2] (0, 1) -- +(0, .5) -- +(.5, .5) -- +(.5,0) -- cycle;
    \fill[DarkGreen!0] (.5, 1) -- +(0, .5) -- +(.5, .5) -- +(.5,0) -- cycle;
    \fill[DarkGreen!15] (1, 1) -- +(0, .5) -- +(.5, .5) -- +(.5,0) -- cycle;
    \fill[DarkGreen!63] (1.5, 1) -- +(0, .5) -- +(.5, .5) -- +(.5,0) -- cycle;
    % 2nd row
    \fill[DarkGreen!6] (0, .5) -- +(0, .5) -- +(.5, .5) -- +(.5,0) -- cycle;
    \fill[DarkGreen!10.9] (.5, .5) -- +(0, .5) -- +(.5, .5) -- +(.5,0) -- cycle;
    \fill[DarkGreen!11.3] (1, .5) -- +(0, .5) -- +(.5, .5) -- +(.5,0) -- cycle;
    \fill[DarkGreen!21] (1.5, .5) -- +(0, .5) -- +(.5, .5) -- +(.5,0) -- cycle;
     % 3rd row
    \fill[DarkGreen!11.5] (0, 0) -- +(0, .5) -- +(.5, .5) -- +(.5,0) -- cycle;
    \fill[DarkGreen!15] (.5, 0) -- +(0, .5) -- +(.5, .5) -- +(.5,0) -- cycle;
    \fill[DarkGreen!15.1] (1, 0) -- +(0, .5) -- +(.5, .5) -- +(.5,0) -- cycle;
    \fill[DarkGreen!19] (1.5, 0) -- +(0, .5) -- +(.5, .5) -- +(.5,0) -- cycle;
    
    \draw[step=0.5cm,color=black!30] (0,0) grid (2.0,1.5);
    \node at (0.25, 0.75) {\small $\A_2$};
        
\end{tikzpicture}
    \caption{(a) $\A_1$ designs a maze for $\A_2$ to navigate in and collect a reward in the top right corner. 
    $\A_2$ behaves differently, i.e.\ chooses a different path, depending on the maze created by $\A_1$.
    (b)~The mean reward function computed using Bayesian IRL \citep{ramachandran2007BIRL} when observing $\A_2$ navigate in each of the three mazes. Dark colours denote higher estimated rewards.}
    %Light colours denote lower estimated rewards and dark colours higher estimated rewards.} (details in Appendix~\ref{appendix:details_experiments})
    \label{figure:introductory_example}
\end{figure}
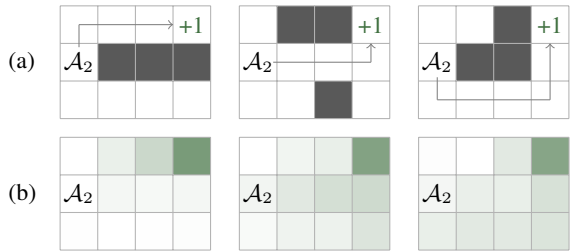

\paragraph{Online Learning.}
As the game progresses, the learner interacts with the expert in a series of episodes, thereby collecting a stream of observations. Then, in order to extract more information as well as to improve cooperation in the next episode, the learner may want to leverage the observations up to episode $t$ to learn about the joint reward function and to inform its decisions in episode $t+1$. 
Naturally, since the learner \emph{actively influences} the demonstrations by the expert, we ask ourselves whether demonstrations under some environment dynamics $\trans_{\pi^1}$ are more informative than others. 
In particular, how much more information (if any) can be gained from demonstrations in unseen environments?  
In the following, we will address these questions both theoretically and empirically. 

%We will see that learning about the reward function in the interactive online setting that we study in this paper,  to learn more efficiently about the reward function than in standard IRL setting due to the expert providing demonstrations under different environment dynamics. 

% We pose the question of how valuable demonstrations of near-optimal behaviour (with respect to the same reward function) are. In Figure 1c) we see that the combination of ...
% It may indeed be interesting to include the cases where we combine observations 1&2, 1&3, and 2&3. 

\section{Cooperating with Optimal Agents}\label{section:cooperating_with_optimal}
Here we consider the case when $\A_2$ responds optimally to the commitment of $\A_1$. %\footnote{If there exist multiple optimal responses, we assume that $\A_2$ chooses one of these uniformly at random.}
In Section~\ref{subsection:IRL_LP}, we characterise the set of \emph{feasible} reward functions, i.e.\ those that are consistent with observed responses, and prove the existence of ideal (reward) learning environments. 
We then describe an algorithm that is no-regret under an assumption on the identifiability of suboptimal behaviour in Section~\ref{subsection: algo-for-opt-responses}. The omitted proofs from this section can be found in Appendix~\ref{appendix:proofs_optimal}.
%In this context, we also show that for a prospective reward function $\r$, there exists an optimal commitment strategy that can be computed efficiently. 

\subsection{Learning from Optimal Responses}\label{subsection:IRL_LP}
For our theoretical analysis, we focus on the full information setting in which $\A_1$ observes the fully specified policy played by the expert at the end of each episode.
In a first step, we define a \emph{feasible} reward function under $(\pi^1, \pi^2)$ as a reward function for which $\A_2$'s response to the commitment of $\A_1$ is optimal.
\begin{definition}
We say that a reward function $\r$ is {\em feasible} when observing policy $\pi^2$ in response to $\pi^1$ if $\pi^2$ is optimal in the single-agent MDP $(\S, A_2, \trans_{\pi^1}, \r, \gamma)$.
\end{definition}
We now adapt the standard result by \citet{ng2000IRL} to obtain a characterisation of the set of feasible reward functions under policies $\pi^1$ and $\pi^2$. Here, we let $\succeq$ denote element-wise inequality. 

\begin{theorem}[\cite{ng2000IRL}]\label{lemma:feasible_set_characterization}
Let there be an MDP without reward function $(\S, A_1, A_2, \trans, \gamma)$. A reward function $\r$ is feasible under policies $\pi^1$ and $\pi^2$ if and only if 
%Given $(\S, A_1, A_2, \trans, \gamma)$ and some policy $\pi^1: \S \to \Delta (A_1)$, a policy $\pi^2: \S \to \Delta (A_2)$ is optimal w.r.t.\ $\pi^1$ and $\r$ if and only if the reward function $\r$ satisfies 
\begin{align*}%\label{equation:feasible_set_characterization}
    \big(\trans_{\pi^1, \pi^2} - \trans_{\pi^1, \actb} \big) \big(I - \gamma \trans_{\pi^1, \pi^2}\big)^{-1} \r \succeq 0 \quad \forall   \actb \in A_2,
\end{align*}
where $\trans_{\pi^1, b}$ is the one-step transition matrix under policy $\pi^1$ and action $b \in A_2$. % Here, $\succeq$ denotes element-wise inequality. 
% where $\trans_{\pi^1, \actb}$ is the one step transition matrix defined as $\trans_{\pi^1} (s' \mid s, \actb)$ for all $s, s' \in \S$, and $\succeq$ denotes component-wise inequality.
\end{theorem}
% \begin{proof}
% Substituting $\trans$ by $\trans_{\pi^1}$ in the proof by \citet{ng2000IRL} readily implies Theorem \ref{lemma:feasible_set_characterization}.
% \end{proof}
%The proofs of all results are collected in Appendix~\ref{appendix:proofs}.

Since $\A_1$ and $\A_2$ repeatedly interact in a series of episodes, a reward function is feasible after $t$ episodes if and only if it is feasible under all policies $\pi^1_1, \dots, \pi^1_t$ and corresponding responses $\pi^2_1, \dots, \pi^2_t$. As an immediate consequence of Theorem~\ref{lemma:feasible_set_characterization}, we then obtain the following characterisation of reward functions that are feasible under multiple observations.
%In other words, a feasible reward function must explain $\A_2$'s behaviour in every environment $\trans_{\pi^1_1}, \trans_{\pi^1_2}, \dots$ 
\begin{corollary}\label{corollary:feasible_set}
Let there be an MDP without reward function $(\S, A_1, A_2, \trans, \gamma)$. A reward function $\r$ is feasible when observing policies $(\pi^1_1, \pi^2_1), \dots, (\pi^1_t, \pi^2_t)$ if and only if
%policy pairs $(\pi^1_1, \pi^2_1), \dots, (\pi^1_t, \pi^2_t)$ if and only if 
\begin{align*}%\label{equation:linear_constraints}
\begin{split}
    & \big(\trans_{\pi^1_1, \pi^2_1} - \trans_{\pi^1_1, \actb} \big) \big(I - \gamma \trans_{\pi^1_1, \pi^2_1}\big)^{-1} \r \succeq 0 \quad \forall \actb \in A_2 , \\
    & \qquad \qquad \qquad \dots \\
    &  \big(\trans_{\pi^1_t, \pi^2_t} - \trans_{\pi^1_t, \actb} \big) \big(I - \gamma \trans_{\pi^1_t, \pi^2_t}\big)^{-1} \r \succeq 0 \quad \forall \actb \in A_2 .
\end{split}
\end{align*}
\end{corollary}
% \begin{proof}%\mgcomment{this could be cut, if above comment is applied}
% This follows from Theorem~\ref{lemma:feasible_set_characterization}.
% \end{proof}
We denote the set of reward functions that satisfy these constraints by $\feasible_t = \feasible((\pi^1_1, \pi^2_1), \dots, (\pi^1_t, \pi^2_t))$. 
% We denote the corresponding set of feasible reward functions by $\feasible_t = \feasible((\pi^1_1, \pi^2_1), \dots, (\pi^1_t, \pi^2_t))$. 

%We denote the set of reward functions that are feasible under observations $(\pi^1_1, \pi^2_1), \dots, (\pi^1_t, \pi^2_t)$, i.e.\ reward functions $\r$ satisfying constraints \eqref{equation:linear_constraints}, by $\feasible_t = \feasible((\pi^1_1, \pi^2_1), \dots, (\pi^1_t, \pi^2_t))$. 

%\subsubsection{Theoretical Analysis of the IRL Problem}
%We will now prove the existence of ideal environment dynamics $\mathcal{T}: \S \times A_2 \to \Delta(\S)$ in which optimal demonstrations of agent $\A_2$ become \emph{maximally informative}. In particular, while estimating the true reward function $\truereward$ precisely is generally impossible in standard IRL scenarios, we will see that there always exist environment dynamics that enable a learner to identify the true reward function (up to positive affine transformations) by interacting with an expert in a single episode. The purpose of this section is therefore to provide the reader with a theoretically grounded sense of the nature and potential of the interactive two-agent IRL problem studied in this paper before more practical approaches are being presented in the following sections. 

% always satisfy the constraints of Theorem~\ref{lemma:feasible_set_characterization} and Corollary~\ref{corollary:feasible_set} and thus 
The IRL problem is an inherently ill-posed problem as degenerate solutions such as constant reward functions explain any observed behaviour. 
In fact, we see that any reward function $\smash{\r \in \Reals^{|\S|}}$ is indistinguishable from its positive affine transformations $\Aff (\r) = \{ \lambda_1 \r + \lambda_2 \mathbf{1} \colon \lambda_1 \geq 0, \lambda_2 \in \Reals\}$. 
% be the set of all positive affine transformation of $\r \in \Reals^{|\S|}$. 
\begin{lemma}\label{lemma:minimal feasible set}
If $\A_2$ responds optimally to the commitment of $\A_1$, any reward function $\r$ is indistinguishable from its positive affine transformations, i.e.\ $\r$ is feasible iff every $\bar \r \in \Aff (\r)$ is feasible.
\end{lemma}
% Let $\truereward$ denote the actual reward function in the MDP $(\S, A_1, A_2, \trans, \truereward, \gamma)$. 
In particular, Lemma~\ref{lemma:minimal feasible set} states that all positive affine transformations of the true reward function $\truereward$ are always feasible.\footnote{We generally denote the true underlying reward function by~$\truereward$. Note that $\truereward$ is unknown to and unobserved by $\A_1$.} 
% \textcolor{purple}{However, it is also easy to check that any reward function in $\Aff(\truereward)$ induces the same optimal joint policy as $\truereward$. In particular, as we will see later, any optimal joint policy yields an optimal commitment strategy for $\A_1$ so that finding $\Aff(\truereward)$ is sufficient for optimally solving the IRL problem.}
However, since any reward function in $\Aff(\truereward)$ induces the same optimal (joint) policy, finding it is sufficient for optimally solving the IRL problem. %\textcolor{purple}{, as we will later see that any optimal joint policy yields an optimal commitment strategy for $\A_1$.  }

Perhaps surprisingly, we find that if $\A_1$'s policies can induce any transition matrix for $\A_2$, then there exists a policy $\pi^1$ such that its optimal response $\pi^2(\pi^1)$ can only be explained by positive affine transformations of the true reward function.
\begin{theorem}\label{proposition:ideal_environment}
(A) If $\A_2$ responds optimally and (B) if for all $\mathcal{T}: \S \times A_2 \to \Delta(\S)$ there exists ${\pi^1}$ such that $\trans_{\pi^1} \equiv \mathcal{T}$, then there exists a policy ${\pi^1}$ with optimal response ${\pi^2}$ such that the feasible set of reward functions under $(\pi^1, \pi^2)$ is given by $\minset(\truereward)$, i.e.\ $\feasible ((\pi^1, \pi^2)) = \minset(\truereward)$. 
\end{theorem}
To emphasise the interpretation and relevance of Theorem~\ref{proposition:ideal_environment} in the standard single-agent IRL setting, we can also rephrase Theorem~\ref{proposition:ideal_environment} as follows:
\begin{remark}
For any state space $\S$, action space $A$, reward function $\truereward$ and discount factor $\gamma \in [0, 1)$, there exists a transition matrix $\mathcal{T}:\S \times A \to \Delta(\S)$ such that the optimal policy $\pi$ in $(\S, A, \mathcal{T}, \truereward, \gamma)$ uniquely characterises $\truereward$ up to positive affine transformations. 
\end{remark}
% [Emphasise that this is an existence proof (and nothing more). See the Corollary below for something more practical.] 
% Here, we would like to comment that there has been recent interest in the problem of reward identifiability in IRL \citep{kim2021reward, cao2021identifiability}. While these works ask the question whether we can identify rewards given a certain MDP model, the above remark makes the statement that given a reward function, we can (in theory) always find an MDP such that $\truereward$ is identifiable. 

% \textcolor{purple}{There has also been recent interest in the problem of reward identifiability in IRL. \citet{kim2021reward} analyse what MDP models allow for a reward function to be identifiable up to some equivalence relation. \citet{cao2021identifiability} specifically consider entropy regularised MDPs and characterise the reward functions inducing specific policies. While both paper take a MDP-first approach, inspired by our interactive IRL formulation, we will show that for any reward function defined over states there exists a single-agent (stochastic) MDP such that reward function is uniquely characterised up to its positive affine transformations (according to Lemma~\ref{lemma:minimal feasible set} in stochastic MDPs with optimal demonstrations that is the smallest achievable equivalence class).}

This leads to the following corollary, which shows that it is possible to check in a single episode whether any given reward function is an affine transformation of $\truereward$. 
%The proof can be found in Appendix~\ref{appendix:proofs}. 
\begin{corollary}\label{theorem:verifying_r_transform_of_true_reward}
Under Assumptions (A) and (B) of Theorem~\ref{proposition:ideal_environment}, the learner can verify in any episode whether a reward function $\r$ is a positive affine transformation of the unknown and unobserved reward function $\truereward$.
\end{corollary} 
% The assumption that $\A_1$ can create any environment dynamics is very strong. However, we notice that while retrieving the set $\minset(\truereward)$ is clearly desirable, it is generally not necessary in order to cooperate optimally as other reward function may as well induce optimal cooperative behaviour. Thus, milder assumptions may be sufficient to learn about the reward function so that $\A_1$ is an optimal partner to $\A_2$. 
We have shown that for any reward function $\truereward$ there exists an environment $\mathcal{T}: \S \times A_2 \to \Delta(\S)$ such that the optimal policy with respect to $\mathcal{T}$ and $\truereward$ characterises $\truereward$ up to positive affine transformations (Theorem~\ref{proposition:ideal_environment}). This implied that the learner, without knowledge of $\truereward$, can verify whether a reward function is element in $\minset(\truereward)$ by playing a specific policy (Corollary~\ref{theorem:verifying_r_transform_of_true_reward}). However, the assumption that $\A_1$ can create any environment dynamics is very strong and we notice that, while retrieving the set $\minset(\truereward)$ is clearly desirable, it is generally not necessary in order to cooperate optimally as other reward functions may also induce optimal behaviour. 
Thus, milder assumptions may be sufficient to learn about the reward function so that $\A_1$ is an optimal partner to $\A_2$. 
In the following, we propose an algorithm that learns about the reward function by adaptively designing environments and that is no-regret under mild assumptions.

%Thus, milder assumptions may be sufficient to learn about the reward function so that $\A_1$ is an optimal partner to $\A_2$.

\subsection{An Algorithm for Interactive IRL}\label{subsection: algo-for-opt-responses}
% We now present an online algorithm that is able to obtain information about the reward function by adaptively choosing the next policy and is no-regret under much milder assumptions.

We now present an online algorithm for learning from and cooperating with an optimally responding agent $\A_2$ when agent $\A_1$ gets to observe the fully specified policy of $\A_2$ at the end of each episode. 
Note that we can always restrict the space of reward functions to the $|\S|$-dimensional unit simplex $\Delta(\S)$ as any positive affine transformation of $\r \in \Delta(\S)$ is equivalent to $\r$ in the sense that they are feasible under the same observations and induce the same optimal (joint) policies (Lemma~\ref{lemma:minimal feasible set}).
Now, as the constraints characterising the feasible set $\feasible_t = \feasible((\pi^1_1, \pi^2_1), \dots, (\pi^1_t, \pi^2_t))$ are linear in the reward function (Corollary~\ref{corollary:feasible_set}), we can use a Linear Program (LP) to find a reward function in $\feasible_t \cap \Delta(\S)$. 
Let $\mathcal{C}((\pi^1_1, \pi^2_1), \dots, (\pi^1_t, \pi^2_t))$ denote the set of constraints induced by $(\pi^1_1, \pi^2_1), \dots, (\pi^1_t, \pi^2_t)$. 
% We denote the set of constraints induced by $(\pi^1_1, \pi^2_1), \dots, (\pi^1_t, \pi^2_t)$ by $\mathcal{C}((\pi^1_1, \pi^2_1), \dots, (\pi^1_t, \pi^2_t))$. 
In episode $t+1$, we then sample an $|\S|$-dimensional objective function $c$ uniformly at random and solve the following LP:
\begin{align}\label{equation:linear_program}
    \max_{\r \in \Delta^{|S|}} c^\top \r \ \text{subject to } \mathcal{C}((\pi^1_1, \pi^2_1), \dots, (\pi^1_t, \pi^2_t)).
\end{align}

\begin{algorithm}[tb]
    \caption{Interactive IRL via Linear Programming}
    \label{algorithm:full_info_opt_behaviour}
\begin{algorithmic}[1]
\STATE {\bfseries input:} $(\S, A_1, A_2, \trans, \gamma)$, initial policy $\pi^1_1$
\FOR{$t = 1, 2, \dots $}
\STATE commit to policy $\pi^1_t$ % and observe response $\pi^2_t$
\STATE observe response $\pi^2_t$
\STATE get constraints $\mathcal{C}_t = \mathcal{C}((\pi^1_1, \pi^2_1),\dots, (\pi^1_t, \pi^2_t))$
\STATE sample objective vector $c$ uniformly at random
\STATE find solution $\r_t \in \feasible_t$ of LP \eqref{equation:linear_program} for $\mathcal{C}_t$ and $c$
\STATE compute $\pi^1_{t+1} \in \optPi_1(\r_t)$
\ENDFOR
\end{algorithmic}
\end{algorithm}

% Intuitively, we are then more likely to choose a vertex in the feasible set that ''covers`` a large part of its boundary. [...]\mgcomment{Maybe: Intuitively, we are then more likely to choose a vertex in the feasible set of reward functions that invokes the same optimal policies as many other reward functions on the boundary of the feasible set.}
In the unlikely event that the LP computes the constant reward function in $\Delta(\S)$, we resample the objective $c$ and solve the LP again. 
Given a prospective reward function $\r$, we then want to compute an optimal commitment strategy in $(\S, A_1, A_2, \trans, \r, \gamma)$. 
% Such an optimal commitment strategy is easy to find when $\A_2$ responds optimally.
We see that if $\A_2$ responds optimally, it suffices to find an optimal joint policy as it yields an optimal commitment strategy for $\A_1$. 
\begin{lemma}\label{lemma:optimal_joint_policy_optimal}
Let $(\bar \pi^1, \bar\pi^2)$ % \in \argmax_{\pi^1, \pi^2} V_{\pi^1, \pi^2}$ 
be an optimal joint policy. If agent $\A_2$ responds optimally to the commitment of $\A_1$, then $V_{\bar \pi^1, \pi^2(\bar \pi^1)} = V_{\bar \pi^1, \bar \pi^2}$. In particular, this entails that $\max_{\pi^1} V_{\pi^1, \pi^2(\pi^1)} = \max_{\pi^1, \pi^2} V_{\pi^1, \pi^2}$.
\end{lemma}
Note that an optimal joint policy and thus an optimal commitment strategy for $\A_1$ can be computed in time polynomial in the number of states and actions. In episode $t+1$, the algorithm then commits to a policy $\pi^1_{t+1} \in \optPi_1(\r)$, where $\r$ is the solution of the LP \eqref{equation:linear_program} and $\optPi_1(\r)$ is the set of optimal commitment strategies under $\r$.
A description of this approach is given by Algorithm~\ref{algorithm:full_info_opt_behaviour}. 
%We write $\optPi_1(\r)$ for the set of optimal commitment strategies under reward function $\r$ and defer the discussion of how to compute an optimal commitment strategy under optimal responses to Section~\ref{subsection:planning_optimal}. 
In fact, we can show that Algorithm~\ref{algorithm:full_info_opt_behaviour} is \emph{no-regret} under the assumption that reward functions that induce suboptimal joint policies are identifiable in the sense that these also induce suboptimal responses. 
\begin{proposition}\label{lemma:algorithm_converges}
Suppose that for any non-constant reward function $\r \in \Delta(\S)$ it holds that if an optimal joint policy $(\pi^1, \pi^2)$ under $\r$ is suboptimal under $\truereward$, then in return there exists an optimal response $\pi^2(\pi^1)$ under $\truereward$ that is suboptimal under $\r$. Moreover, assume that $\A_2$ responds optimally and breaks ties between equally good policies uniformly at random. Then, the average regret suffered by Algorithm~\ref{algorithm:full_info_opt_behaviour} converges to zero almost surely. 
\end{proposition}
\begin{proof}[Proof Sketch]
The proof relies on a finite cover of the space of reward functions. We can show that in every step of the algorithm either an optimal policy was played (generating no regret) or with positive probability the reward functions in at least one of the sets of the cover become infeasible - thus ultimately reducing the set of feasible reward functions to only those that yield optimal policies.
\end{proof}

% \subsection{Planning with Optimal Agents}\label{subsection:planning_optimal} 

\section{Cooperating with Suboptimal Agents}
\label{sec:suboptimal}
We now consider the case when $\A_2$ responds suboptimally according to some behavioural model such as Boltzmann-rational policies. 
Section~\ref{subsection:bayesian_IRL} extends the Bayesian IRL formulation to our setting and Section~\ref{subsection:planning_suboptimal} analyses the problem of computing optimal commitment strategies when $\A_2$ is playing suboptimally. The omitted proofs from this section can be found in Appendix~\ref{appendix:proofs_suboptimal}.

\subsection{Learning from Suboptimal Responses}\label{subsection:bayesian_IRL} % Extending Bayesian IRL
When demonstrations are possibly {suboptimal}, it is natural to take a Bayesian perspective \citep{ramachandran2007BIRL} as it provides a principled way to reason under uncertainty. 
Moreover, the Bayesian approach naturally extends to the {partial information} setting, where only trajectories generated by both agents' policies are available for learning. 
We assume that $\A_2$ responds with Boltzmann-rational policies with \emph{unknown} inverse temperature 
$\beta$\footnote{Note that any other parameterised behavioural model could also be modeled by this Bayesian formulation.}
and adapt the Bayesian IRL formulation to our setting. 
% We adapt the Bayesian formulation of the IRL problem to our setting and assume that $\A_2$ responds with Boltzmann policies with unknown inverse temperature $\beta$.\footnote{Note that any other behavioural model for $\A_2$ can also be modeled by our Bayesian formulation.}
%However, any other behavioural model for $\A_2$ can also be modeled by our Bayesian formulation.
Suppose that in the first $t$ episodes $\A_1$ observes $(\pi^1_1, \tau_1), \dots, (\pi^1_t, \tau_t)$, where $\tau_i$ % = % (s_i^H, a_i^H, b_i^H)$ is the trajectory of length $H+1$ 
%(s_{i,0}, a_{i, 0}, b_{i,0}, \dots, s_{i, H}, a_{i, H}, b_{i, H})$ 
is the trajectory generated by $\A_1$'s policy $\pi^1_i$ and $\A_2$'s response $\pi^2_i(\pi^1_i)$ for $i \in [t]$.\footnote{For notational conciseness, we assume here that the length of a trajectory is fixed across all episodes.}
Bayesian IRL aims to estimate the posterior
\begin{align*}
    & \P (\r, \beta \mid (\pi^1_1, \tau_1), \dots, (\pi^1_t, \tau_t)) \\[0.3em]
    & \qquad \qquad \quad = \frac{\P( (\pi^1_1, \tau_1), \dots, (\pi^1_t, \tau_t) \mid \r, \beta ) \P(\r) \P(\beta)}{\P((\pi^1_1, \tau_1), \dots, (\pi^1_t, \tau_t))},
\end{align*}
given priors $\P(\r)$ and $\P(\beta)$ over reward functions and inverse temperatures, respectively. We notice that the observations $(\pi^1_1, \tau_1), \dots, (\pi^1_t, \tau_t)$ are conditionally independent under measure $\P(\cdot \mid \r , \beta)$. As a result, we can express their likelihood as $\P((\pi^1_1, \tau_1), \dots, (\pi^1_t, \tau_t)\mid \r, \beta ) = \prod_{i=1}^t \P((\pi^1_i, \tau_i) \mid \r, \beta)$.
The likelihood for each observation $(\pi^1_i, \tau_i)$ can then be computed as
\begin{align*}
    \P ((\pi^1_i, \tau_i) \mid \r, \beta) & = \prod_{h = 0}^H \pi^2(\actb_{i, h} \mid s_{i, h}, \pi^1_i, \r, \beta) \\
    & \propto \exp \big(\beta \sum_{h=0}^H  Q_{\pi^1_i}^* ( s_{i, h}, \actb_{i, h}, \r) \big).
\end{align*}
%In particular, as the observations $(\pi^1_1, \tau_1), \dots, (\pi^1_t, \tau_t)$ are conditionally independent under measure $\P(\cdot \mid \r, \beta)$, we could apply standard single-agent Bayesian IRL methods to the observation of each episode separately.  
%Since the observations $(\pi^1_1, \tau_1), \dots, (\pi^1_t, \tau_t)$ are conditionally independent given $\r$ and $\beta$, we can compute the likelihood separately for each demonstration. 
%separately in order to estimate the full likelihood $\P((\pi^1_1, \tau_1), \dots, \pi^1_t, \tau_t) \mid \r, \beta))$ (and thus the posterior). 
The Bayesian method we employ generates samples from the posterior via Markov Chain Monte Carlo (MCMC), similarly to \citep{ramachandran2007BIRL, rothkopf2011preference}. At a high level, we employ a Metropolis-Hastings algorithm on the reward simplex, with a uniform prior on the reward function and an exponential prior on the inverse temperature (see Algorithm~\ref{algorithm:Bayesian_IRL_MCMC} in Appendix~\ref{appendix:details_experiments}).  
%or uses gradient methods to perform maximum-a-posteriori estimation \cite{choi2011MAP-BIRL}. Recently, other scalable Bayesian approaches have been proposed, e.g.\ \cite{brown2019machine, chan2021scalable}.
%Note that typically the inverse temperature $\beta$ is assumed to be known in the aforementioned methods. However, MCMC methods can account for an unknown $\beta$ as they directly sample from the posterior. 
%In our experiments, we will use an adaptation of the MCMC approaches by \citet{ramachandran2007BIRL} and \citet{rothkopf2011preference} to our setting (). % (see Section~\ref{sec:experiments}).

\subsection{Planning with Suboptimal Agents}\label{subsection:planning_suboptimal}

Prior work on computing optimal commitment strategies in stochastic games typically assumes that the follower is responding optimally \citep{letchford2012computing, vorobeychik2012computing}. In this section, we analyse optimal commitment strategies for the \emph{cooperative} Stackelberg game from Section~\ref{section:setting} when agent $\A_2$, i.e.\ the follower, responds \emph{suboptimally} according to some behavioural model, e.g.\ Boltzmann-rational policies or $\varepsilon$-greedy policies. For this, the concept of dominating policies play a crucial role.
\begin{definition}
A policy $\pi^1$ is {\em dominating} if $V_{\pi^1, \pi^2(\pi^1)} (s) \geq V_{\bar \pi^1, \pi^2(\bar \pi^1)}(s) $ for all policies $\bar \pi^1$ and states $s \in \S$. % of agent $\A_1$.
\end{definition}
The existence of dominating policies is closely linked to our capacity to compute an optimal commitment strategy efficiently as it is a key requirement for dynamic programming. 
We show that if $\A_2$ plays proportionally with respect to the expected value of taking an action, there may not exist dominating policy for $\A_1$ to commit to.

\begin{theorem}[]\label{lemma:no_dominating_policy}
If $\pi^2 (\actb \mid s) \propto f(Q_{\pi^1}^*(s, \actb))$ for any strictly increasing function $f:[0, \infty) \to [0, \infty)$, then a dominating commitment strategy for agent $\A_1$ may not exist.
\end{theorem}

In particular, this means that if $\A_2$ plays Boltzmann-rational policies, a dominating commitment strategy may fail to exist. 
Note that Theorem \ref{lemma:no_dominating_policy} generally only holds for {\em strictly} increasing functions $f$, as, for instance, there always exists a dominating commitment strategy when $\A_2$ plays uniformly at random. 
However, even for behavioural models as simple as $\varepsilon$-greedy, we see that a dominating commitment strategy does not necessarily exist.
% In particular, this means that if $\A_2$ plays Boltzmann-rational policies, a dominating commitment strategy may fail to exist. 
% Note that Theorem \ref{lemma:no_dominating_policy} generally only holds for {\em strictly} increasing functions $f$, as, for instance, there always exists a dominating commitment strategy when $\A_2$ plays uniformly at random. However, even for behavioural models as simple as $\varepsilon$-greedy policies, we see that a dominating commitment strategy may not exist.

\begin{lemma}[]\label{lemma:eps_greedy_dominating}
If $\A_2$ plays $\varepsilon$-greedy, a dominating commitment strategy for $\A_1$ may not exist. 
\end{lemma}

Despite these difficulties, we provide algorithms to approximate optimal commitment strategies for the case of Boltzmann-rational responses (Algorithm~\ref{algorithm:value_iteration_boltzmann}) and $\varepsilon$-greedy responses (Algorithm~\ref{algorithm:value_iteration_greedy}), which can be found in Appendix~\ref{appendix:approx_algorithms}. The proposed methods correspond to approximate value iteration algorithms that keep track of two value functions, each modelling one agent.
% However, since a dominating commitment strategy may fail to exist, we could not provide is no converge guarantee for our proposed algorithms. 
% The proposed approximate value iteration algorithms are modified so as to keep track of two value functions, each corresponding to one agent. 
% In essence, our adaptations of the value iteration algorithm ignore the fact that a dominating strategy may not exist and could thus converge to suboptimal solutions. 
We include an empirical evaluation of the proposed algorithms in Appendix~\ref{appendix:approx_evaluation}, which demonstrates that accounting for the suboptimal nature of $\A_2$ reliably improves performance. 

% Additional experiments also show that an incorrect estimate of $\A_2$'s behavioural model can be hazardous and that assuming optimality is more robust. 

% \input{sections/algorithm_approx_vi_boltz}
\section{Experiments}\label{sec:experiments}

%Our experiments show that the interactive nature of our setting enable us to learn the reward function efficiently and precisely in only few episodes.
%In particular, we investigate how much the learner, i.e.\ agent $\A_1$, benefits from observing $\A_2$ in unseen environments. 
In our experiments, we investigate how much the learner benefits from repeatedly interacting with the expert. 
To address this question and emphasise the potential benefit of demonstrations in different environments, we include the situation where $\A_1$ only observes the response of $\A_2$ to the initial policy $\pi^1_1$ played by $\A_1$. This resembles the standard IRL setting where we observe the expert only in a single fixed environment $(\S, A_2, \trans_{\pi^1_1}, \r, \gamma)$. 
% demonstrates a task in a fixed environment $(\S, A_2, \trans_{\pi^1_1}, \r, \gamma)$. 

% \paragraph{Setup.}
Here, the initial policy $\pi^1_1$ is chosen uniformly at random. We model the standard IRL setting by repeatedly generating responses of $\A_2$ with respect to $\pi^1_1$, i.e.\ in the implied environment $\trans_{\pi^1_1}$. 
Using these observations, we then estimate the reward function using standard IRL, compute the optimal policy with respect to the estimated rewards, and evaluate the regret of this policy. % evaluate the optimal policy with respect to this estimate. 
In contrast, in the Interactive IRL setting, the learner gets to choose a different policy in subsequent episodes. 
In this case, we report the online regret of the actually played policies, i.e.\ the actual regret of the learner. More details are provided in Appendix~\ref{appendix:details_experiments}.

\subsection{Environments}

\paragraph{Maze-Maker.} In this environment, agents $\A_1$ and $\A_2$ jointly control a cart in a $7\times7$ grid world. 
In this grid world, the doors leading from one cell to the neighbouring ones are locked. However, $\A_1$ can unlock exactly two doors at any time step before they fall shut again. 
Agent $\A_2$ can attempt to move the cart through a door to a neighbouring cell. 
However, when the door is locked, the cart stays where it was. 
% We assume that any attempted move of the cart succeeds with probability $0.8$ and that with probability $0.2$ the cart moves to a random neighbouring cell. 
The agents are tasked with collecting three rewards of different value (+$1$, +$2$, +$3$), which disappear once collected. %are scattered in the grid world and disappear once collected. 
While the expert, $\A_2$, knows where the rewards are placed, the helper, $\A_1$, does not know their location. 
We model this environment as a two-agent MDP with $392$ states ($49 \times 8$) and discount factor $\gamma = 0.9$, where $\A_1$ has six actions (unlocking two out of four doors) and $\A_2$ four actions (moving the cart North, East, South, West). 
An illustration of the environment is given in Figure~\ref{figure:experiments_maze_maker}.

%As we consider a Stackelberg game, $\A_2$ knows beforehand which doors $\A_1$ will unlock. Therefore, $\A_1$ essentially selects a maze layout (see Figure~\ref{subfigure:maze_maker_middle}), which is communicated to $\A_2$ and through which $\A_2$ navigates the cart. 

% \begin{figure}[t]
%      \begin{center}
%          \begin{subfigure}[t]{0.25\textwidth}
%             \resizebox{!}{\columnwidth}{
%              \input{figures/figure_maze_maker_env/maze_maker_left}
%              }
%              \caption{}
%              %\label{fig:y equals x}
%          \end{subfigure}
%          \hspace{0.3cm}
%          \begin{subfigure}[t]{0.25\textwidth}
%             \resizebox{!}{\columnwidth}{
%             \input{figures/figure_maze_maker_env/maze_maker_middle}
%             }
%              \caption{}
%              \label{subfigure:maze_maker_middle}
%          \end{subfigure}
%          \hspace{0.3cm}
%          \begin{subfigure}[t]{0.25\textwidth}
%             \resizebox{!}{\columnwidth}{
%             \input{figures/figure_maze_maker_env/maze_maker_right}
%             }
%              \caption{}
%              %\label{fig:five over x}
%          \end{subfigure}
%      \end{center}
%         \caption{The Maze-Maker environment. a) The initial game setup with initial state in the center and the rewards scattered in different areas of the grid. b) When $\A_1$ commits to a policy it implicitly creates a maze for $\A_2$ to navigate the cart in. c) An exemplary route taken by $\A_2$ in maze~b).}
%         \label{figure:experiments_maze_maker}
% \end{figure}

\begin{figure}[H]
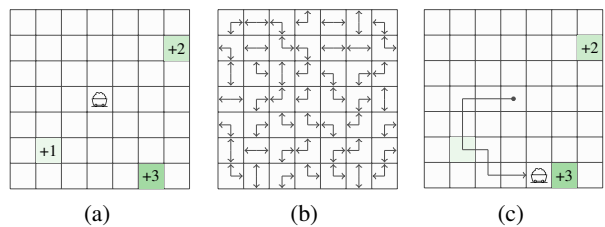

    \subfigure[]{%
    \resizebox{!}{.29\linewidth}{
    \input{figures/figure_maze_maker_env/maze_maker_left}}
    }
    \hfill 
    \subfigure[]{%
    \resizebox{!}{.29\linewidth}{
    \colorlet{DarkGreen}{green!60!black!40}
%SECOND
\begin{tikzpicture}
\fill[gray!2] (0,0) -- +(0, 3.5) -- +(3.5, 3.5) -- +(3.5, 0) -- cycle;

\draw[step=0.5cm,color=black!70] (0,0) grid (3.5,3.5);
% \fill[blue!5] (1.51, 1.51) -- +(0, .48) -- +(.48, .48) -- +(.48,0) -- cycle;

% % ARROWS
% North-East
\foreach \i/\j in {0/0, 1/4, 2/5, 3/0, 4/3, 6/5, 2/2, 3/2, 4/0}
{
    \draw[-to, black!70] (0.25 + \i *0.5, 0.25 +\j * 0.5) -- (0.25 + \i *0.5, 0.48 + \j *0.5);
    \draw[-to, black!70] (0.25 + \i *0.5, 0.25 +\j * 0.5) -- (0.48 + \i *0.5, 0.25 + \j *0.5);
}
% North-South
\foreach \i/\j in {0/1, 0/4, 1/0, 2/4, 5/0, 6/3, 5/6}
{
    \draw[-to, black!70] (0.25 + \i *0.5, 0.25 +\j * 0.5) -- (0.25 + \i *0.5, 0.48 + \j *0.5);
    \draw[-to, black!70] (0.25 + \i *0.5, 0.25 +\j * 0.5) -- (0.25 + \i *0.5, 0.02 + \j *0.5);
}

% North-West
\foreach \i/\j in {6/0, 3/3, 2/3, 3/1, 3/4, 3/6, 5/1, 5/2, 6/4}
{
    \draw[-to, black!70] (0.25 + \i *0.5, 0.25 +\j * 0.5) -- (0.25 + \i *0.5, 0.48 + \j *0.5);
    \draw[-to, black!70] (0.25 + \i *0.5, 0.25 +\j * 0.5) -- (0.02 + \i *0.5, 0.25 + \j *0.5);
}

% South-East
\foreach \i/\j in {0/6, 1/3, 1/2, 2/0, 2/1, 4/1, 5/3, 5/4, 6/1, 6/2}
{
    \draw[-to, black!70] (0.25 + \i *0.5, 0.25 +\j * 0.5) -- (0.25 + \i *0.5, 0.02 + \j *0.5);
    \draw[-to, black!70] (0.25 + \i *0.5, 0.25 +\j * 0.5) -- (0.48 + \i *0.5, 0.25 + \j *0.5);
}

% South-West
\foreach \i/\j in {6/6, 0/2, 0/5, 2/6, 3/5, 4/2, 4/4}
{
    \draw[-to, black!70] (0.25 + \i *0.5, 0.25 +\j * 0.5) -- (0.25 + \i *0.5, 0.02 + \j *0.5);
    \draw[-to, black!70] (0.25 + \i *0.5, 0.25 +\j * 0.5) -- (0.02 + \i *0.5, 0.25 + \j *0.5);
}

% West-East
\foreach \i/\j in {1/1, 0/3, 1/5, 1/6, 4/5, 4/6, 5/5}
{
    \draw[-to, black!70] (0.25 + \i *0.5, 0.25 +\j * 0.5) -- (0.02 + \i *0.5, 0.25 + \j *0.5);
    \draw[-to, black!70] (0.25 + \i *0.5, 0.25 +\j * 0.5) -- (0.48 + \i *0.5, 0.25 + \j *0.5);
}
\end{tikzpicture}}
    }
    \hfill
    \subfigure[]{%
    \resizebox{!}{.29\linewidth}{
    \input{figures/figure_maze_maker_env/maze_maker_right}}
    }
    \caption{The Maze-Maker Environment. (a) The initial game setup with starting position in the center and three rewards scattered across the grid world. (b) When $\A_1$ commits to a policy it implicitly creates a maze for $\A_2$ to navigate the cart in. (c) An exemplary path taken by $\A_2$ in the maze implied by $\A_1$'s policy.}
    \label{figure:experiments_maze_maker}
\end{figure}

\paragraph{Random MDPs.} 
We also randomly generated MDPs with $200$ states and four actions for each agent. We randomly draw the transition dynamics from a Dirichlet distribution, with restrictions on the influence of each agent on the transitions, and the rewards from an i.i.d.\ Beta distribution. The discount factor is set to $\gamma = 0.9$.

\subsection{Results} 

\paragraph{Optimal Responses and Full Information:} In Figure~\ref{subfigure:optimal_maze_maker} and \ref{subfigure:optimal_random_MDPs}, we observe that the per-episode regret suffered by Algorithm~\ref{algorithm:full_info_opt_behaviour} in both environments decreases notably with the number of episodes played. In particular, we see that after only a few episodes the per-episode regret of Algorithm~\ref{algorithm:full_info_opt_behaviour} is significantly lower than for maximum-margin IRL \citep{ng2000IRL} when $\A_1$ only observes the response to the initial policy $\pi^1_1$. This roughly corresponds to the standard IRL setting in which demonstrations are obtained in a single environment only.  
% which roughly corresponds to the regret suffered in the standard IRL setting.
We thus find that the learner significantly benefits from observing $\A_2$'s behaviour in new and different environments, i.e.\ with respect to different policies of $\A_1$. In particular, it appears to be necessary to observe the expert's response to several different policies in order to infer an approximately optimal reward function. The results are averaged over $5$ runs.

\paragraph{Suboptimal Responses and Partial Information}
For the case of suboptimal responses and partial information, we let $\A_2$ respond with Boltzmann-rational policies with inverse temperature $\beta = 10$ in both environments. We assume that the inverse temperature, i.e.\ the optimality of $\A_2$, is unknown to the learner and simulate the partial information setting by generating trajectories according to policies $\pi^1_t$ and $\pi^2_t$ in episode $t$. We let an episode end with probability $1-\gamma$ each time step so that the lengths of observed trajectories are random. %, where we let the episode end with probability $1-\gamma = 0.1$ each time step. 
% , where the length of the episode is random. More precisely, we let an episode end with probability $1-\gamma = 0.1$ each time step.

% Optimal Responses and Full Information

\begin{figure}[t]
    \subfigure[Maze-Maker \label{subfigure:optimal_maze_maker}]{%
      \includegraphics[width=.45\linewidth]{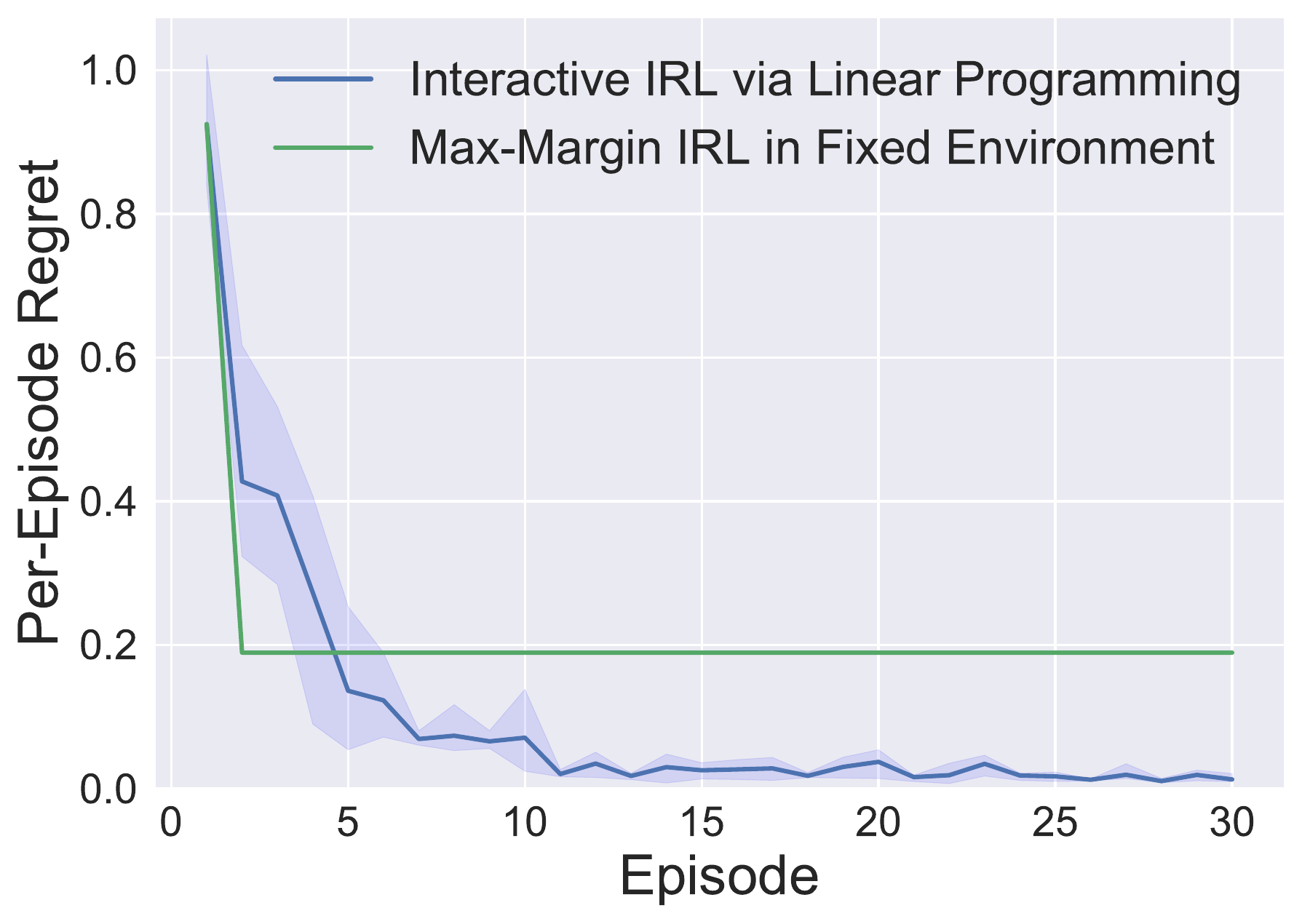}
    }
    \hfill  
    \subfigure[Random MDPs\label{subfigure:optimal_random_MDPs}]{%
      \includegraphics[width=0.45\linewidth]{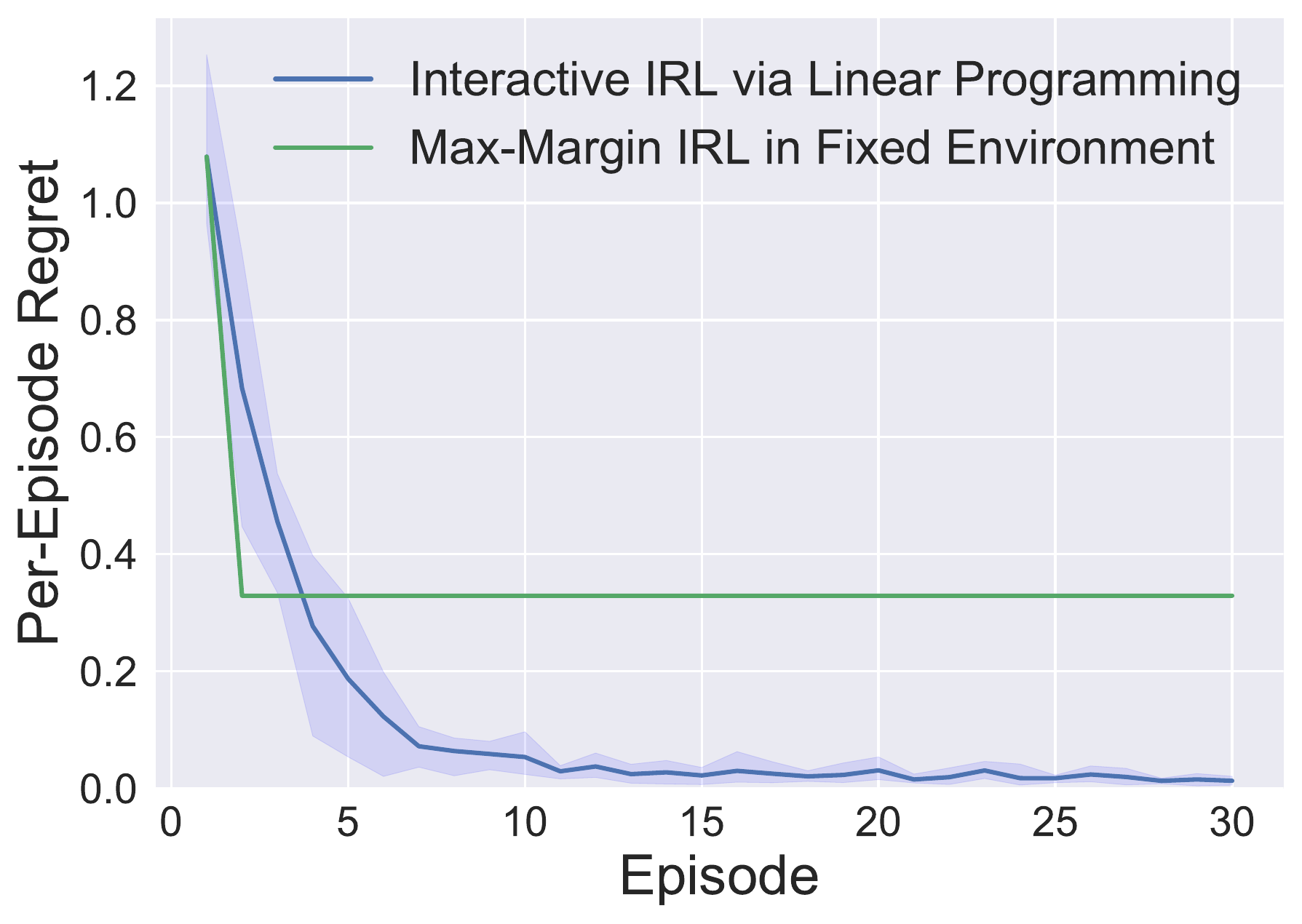}
    }
    \caption{Optimal Responses and Full Information. Blue lines show the per-episode regret $\loss(\pi^1_t)$ of Algorithm~\ref{algorithm:full_info_opt_behaviour}.  
    Green lines correspond to the regret of maximum-margin IRL \citep{ng2000IRL} performed with observation $(\pi^1_1, \pi^2_1)$ only.}
    \label{fig:my_label}
\end{figure}

Figure~\ref{subfigure:suboptimal_maze_maker} and \ref{subfigure:suboptimal_random_MDPs} show that Bayesian Interactive IRL (Algorithm~\ref{algorithm:Bayesian_IRL_MCMC}) reliably improves its estimate of the true reward function with the number of episodes played and that the learner again substantially benefits from observing $\A_2$ act in different environments. While obtaining an increasing amount of trajectories in the \emph{same} environment improves the estimate of the reward function as well, we see that trajectories generated in new environments, i.e.\ with respect to different policies of $\A_1$, yield much more information and thus allow for a better estimate of the unknown reward function.  
%However, as expected, learning from suboptimal responses and partial information proves to be notably harder than for optimal responses and full information. % In particular, it seems that we cannot achieve close to zero regret in few episodes.%\footnote{We are limited in increasing the number of episodes due to poor scalability of Bayesian IRL via MCMC.} 
%Also note that the large fluctuations in regret are most likely due to the MCMC procedure for which we used $25{,}000$ proposals per episode, as well as the \emph{random} trajectory lengths, which vary across episodes and test runs (see Appendix~\ref{appendix:details_experiments}).
The results are averaged over $10$ runs. 

% While $\A_1$ clearly benefits from more
% demonstrations in a fixed environment (see the green lines in Figure~\ref{figure:experiments_suboptimal_responses}), we also find that demonstrations in different environments become increasingly useful compared to demonstrations in a fixed environment the more episodes are played. % stage of the game and the response in the fixed environment is explored exhaustively. 

% Note that the initial values of the blue and green line in Figure~\ref{figure:experiments_suboptimal_responses} correspond to the regret of playing the initial policy $\pi^1_1$ that is chosen uniformly at random in the first episode. 

% Suboptimal Responses and Partial Information

\begin{figure}[t]
\centering
    \subfigure[Maze-Maker \label{subfigure:suboptimal_maze_maker}]{%
      \includegraphics[width=0.45\linewidth]{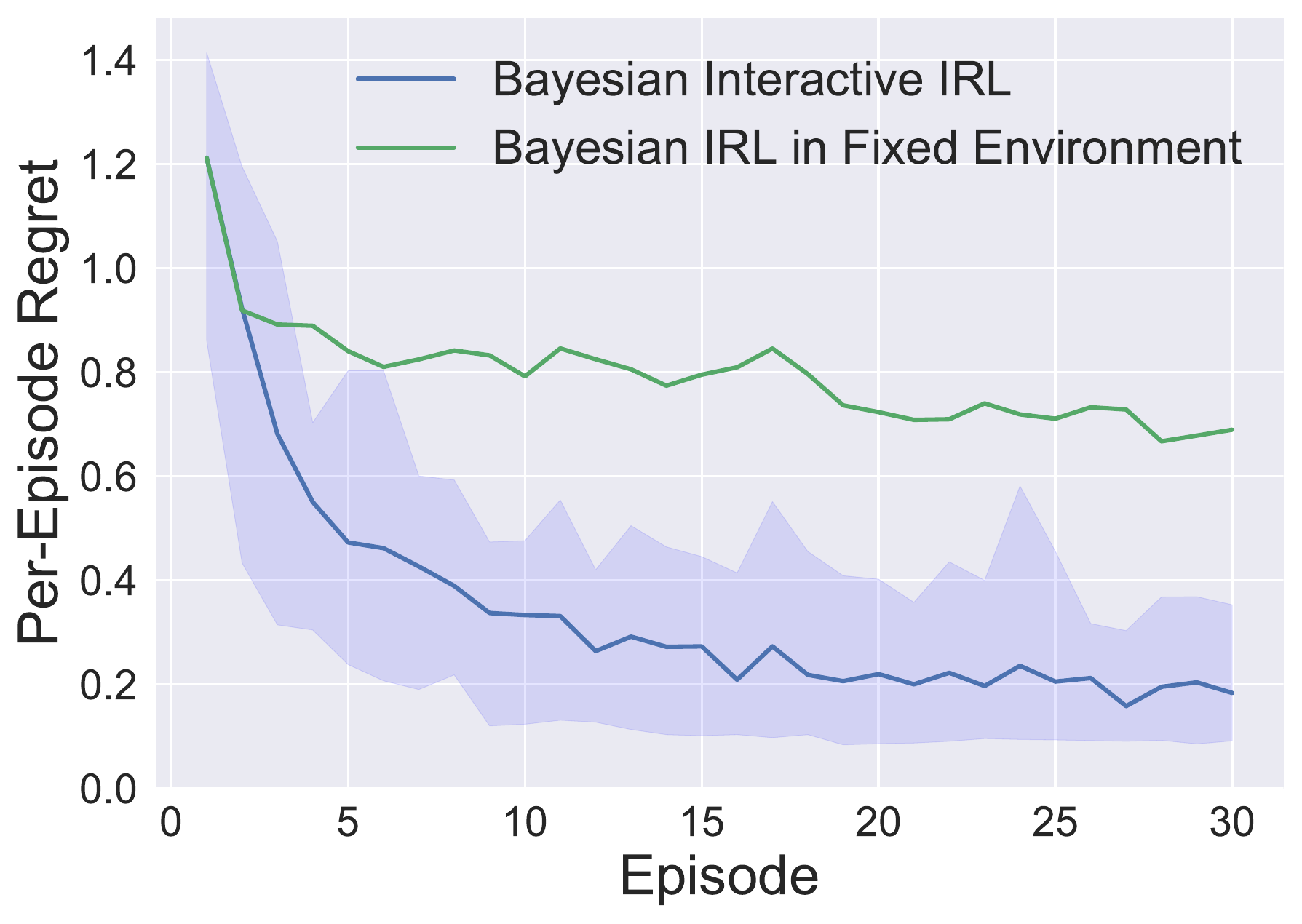}
    }
    \hfill
    \subfigure[Random MDPs\label{subfigure:suboptimal_random_MDPs}]{%
      \includegraphics[width=0.45\linewidth]{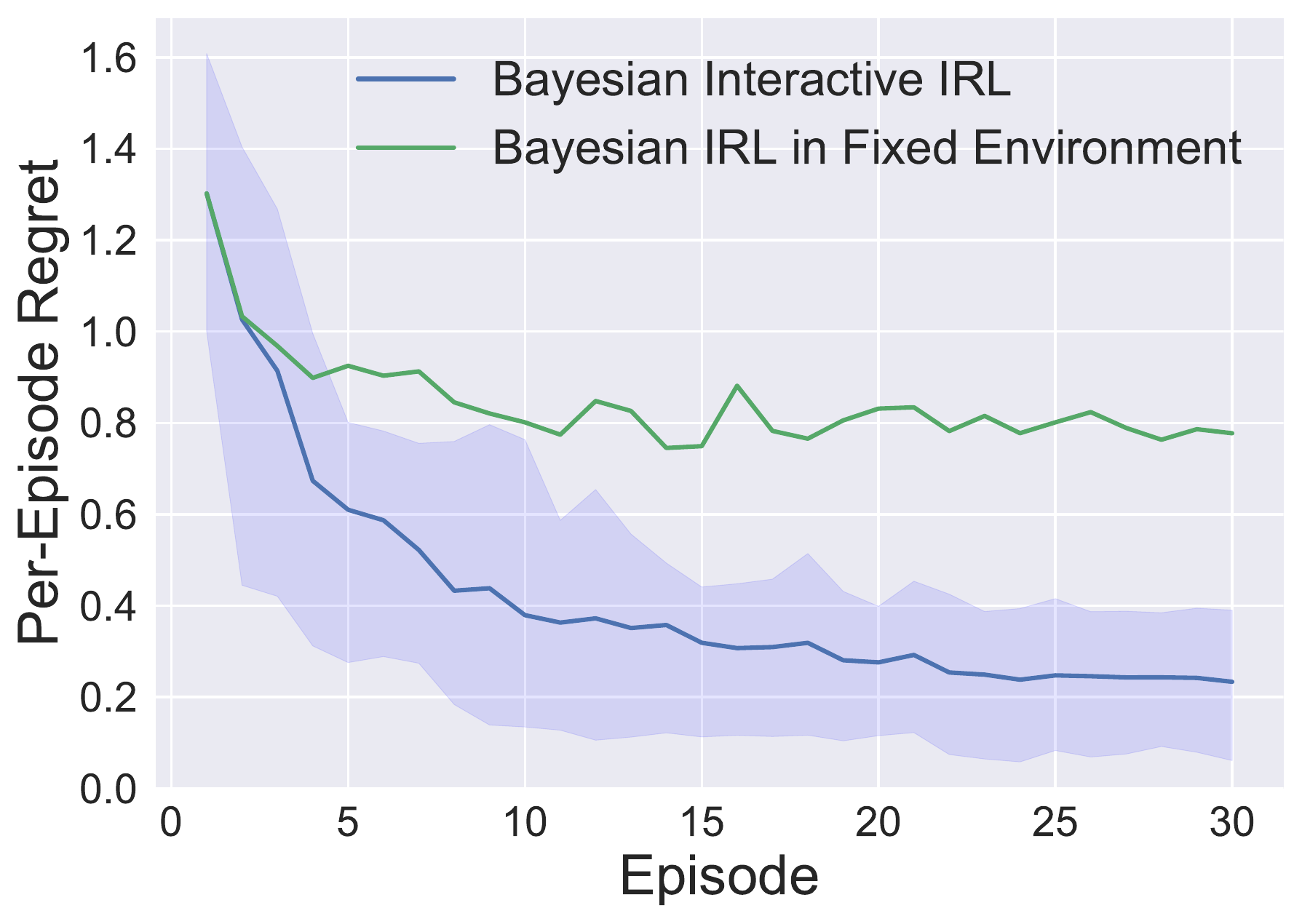}
    }
    \caption{Suboptimal Responses and Partial Information. Blue lines show the per-episode regret of Bayesian Interactive IRL (Algorithm~\ref{algorithm:Bayesian_IRL_MCMC} in Appendix~\ref{appendix:details_experiments}). Green lines refer to Bayesian IRL performed for trajectories repeatedly generated by $\pi^1_1$ and $\pi^2_1$.}
    \label{fig:my_label}
\end{figure}

\section{Discussion and Future Work}\label{section:future_work}
We considered an interactive cooperation problem when the objective is unknown to one of the agents. This can be seen as a two-agent version of the IRL problem, where one agent is actively trying to infer the preferences of the other in order to cooperate. While the classical IRL problem is generally ill-posed, the interactive version that we study here can indeed be solved if the learning agent has sufficient power to affect the transitions. This is supported by both our experimental and theoretical results. In particular, the experiments clearly show that we can more accurately estimate the reward function (and hence collaborate more effectively) if we intelligently probe the other agent's responses.  % This holds in both the full and partial information case. 

An open theoretical question is whether upper and lower problem-dependent bounds on the episodic regret could be obtained in this setting. We presume that such bounds would involve a characterisation of $\A_1$'s power to affect the transitions. 
% FROM REVIEW: For this it may be worth having a look at preference-based RL as see how they handle the noisy dependence on the true and unknown reward function. 
A natural extension of our setting would be the case where $\A_1$ does not reveal its policy to $\A_2$, but instead the latter simply observes the former's actions. %While this might require an elaborate behavioural model for $\A_2$, it could be avoided through an approach similar to that of \citet{radanovic2019learning}, which bounded the behavioural change of $\A_2$.
In future work, it will also be interesting to construct Interactive IRL algorithms that scale to large state spaces (or continuous domains) and test these in real-world applications.

Our observation that reward learning benefits from demonstrations under different environment dynamics also opens up a new and interesting perspective on IRL more generally. While current IRL methods still struggle to learn satisfactory reward functions in certain domains (even with abundant data), it could be promising to try to infer the reward function from demonstrations in slight variations of the target environment (when possible). Moreover, our results suggest that receiving samples under new environment dynamics is generally more valuable than collecting additional samples from the same environment. Thus, such an approach could be useful in domains where resources are limited and samples expensive.

% [Comment on the how we could scale to larger state spaces or even continuous state spaces. Comment on limitations.] 
% We showed that learning the (joint) reward function from interactions allows us not only to infer the reward function more precisely, but also more quickly. One could also think about improving the sample efficiency of IRL by not only designing better IRL algorithms, but also by designing better environments for these demonstrations. 

% [Computing optimal commitment strategies is also important when we want to cooperate with an agent in simultaneous commitment games and $\A_2$ maintains a belief about our next policy.]
% [Experiments with humans would be pretty cool and would actually be easy to do for the maze-maker environment (with partial information and suboptimal responses).]
% Clearly, the Bayesian approach suffers from the same scalability issues as other typical single-agent Bayesian IRL algorithms. However, we gave a first impression on how to extend IRL algorithms, highlighting the additional difficulties (demonstrations in different MDPs), but also that adaptations are not difficult.

% \clearpage
% \newpage 

\bibliographystyle{icml2022}
\bibliography{ref}

\clearpage 
\newpage 

\appendix
\onecolumn

% For all $\r \in \Reals^{|\S|}$, there exists $\pi^1$ such that $\pi^2(\pi^1, \r) = \pi^2(\pi^1, \bar \r)$ if and only if $\bar \r \in \minset(\r)$. 

% This means that if $\r$ is the \textbf{true} reward function, then there exists $\pi^1$ such that the optimal response $\pi^2(\pi^1)$ is the same \textbf{only} for $\bar \r \in \minset(\r)$, i.e.\ $\pi^2(\pi^1) = \pi^2(\pi^1, \bar \r)$ if and only if $\bar \r \in \minset(\r)$. 

% \section{Proofs}\label{appendix:proofs}
% In the following, we collect the proofs. 

\section{Proofs for Section~\ref{section:cooperating_with_optimal}}\label{appendix:proofs_optimal}

\subsection{Proof of Theorem~\ref{lemma:feasible_set_characterization}}

\begin{customthm}{1}
Let there be some MDP without reward function $(\S, A_1, A_2, \trans, \gamma)$. A reward function $\r$ is feasible under policies $\pi^1$ and $\pi^2$ if and only if 
%Given $(\S, A_1, A_2, \trans, \gamma)$ and some policy $\pi^1: \S \to \Delta (A_1)$, a policy $\pi^2: \S \to \Delta (A_2)$ is optimal w.r.t.\ $\pi^1$ and $\r$ if and only if the reward function $\r$ satisfies 
\begin{align*}%\label{equation:feasible_set_characterization}
    \big(\trans_{\pi^1, \pi^2} - \trans_{\pi^1, \actb} \big) \big(I - \gamma \trans_{\pi^1, \pi^2}\big)^{-1} \r \succeq 0 \quad \forall   \actb \in A_2,
\end{align*}
where $\trans_{\pi^1, b}$ is the one-step transition matrix under policy $\pi^1$ and action $b \in A_2$. 
\end{customthm}
\begin{proof}[Proof of Theorem~\ref{lemma:feasible_set_characterization}]
Substituting transition matrix $\trans$ by $\trans_{\pi^1}$ in the proof by \citet{ng2000IRL} readily implies Theorem~\ref{lemma:feasible_set_characterization}. Note that if $\pi^2(s) = \bar \actb$ for all $s \in \S$, the inequality vacuously holds for $\smash{\actb = \bar \actb}$. Thus, in general we obtain $|A_2|-1$ many of the above vector inequalities. 
\end{proof}

\subsection{Proof of Lemma~\ref{lemma:minimal feasible set}}
\begin{customlemma}{1}
If $\A_2$ responds optimally to the commitment of $\A_1$, any reward function $\r$ is indistinguishable from its positive affine transformations, i.e.\ $\r$ is feasible iff every $\bar \r \in \mathcal{A}(\r)$ is feasible.
\end{customlemma}

\begin{proof}[Proof of Lemma \ref{lemma:minimal feasible set}]
% Let $\minset \defn \{ \lambda \truereward + c\mathbf{1} \colon \lambda \in \Reals^+_0, c\in \Reals\}$.
We write $V_{\pi^1, \pi^2}(\r)$ for the value function under joint policy $(\pi^1, \pi^2)$ and reward function $\r$.
The Bellman equation tells us that the value function under $(\pi^1, \pi^2)$ and reward function $\lambda_1 \r + \lambda_2 \mathbf{1} \in \minset(\r)$ is given by 
\begin{align*}
    V_{\pi^1, \pi^2} (\lambda_1 \r + \lambda_2 \mathbf{1}) = (I- \gamma \trans_{\pi^1, \pi^2})^{-1} (\lambda_1 \r + \lambda_2 \mathbf{1}).
\end{align*}
Now, since $\trans_{\pi^1, \pi^2}$ is a stochastic matrix, it is easy to check that $(I- \gamma \trans_{\pi^1, \pi^2})^{-1} \mathbf{1} = (1-\gamma)^{-1} \mathbf{1}$. It then follows that 
\begin{align*}
    V_{\pi^1, \pi^2} (\lambda_1 \r + \lambda_2 \mathbf{1}) = \lambda_1 V_{\pi^1, \pi^2} (\r) + K,
\end{align*}
where $K = \lambda_2 (1-\gamma)^{-1} \mathbf{1}$. Hence, we find that any policy $\pi^2$ that maximises $V_{\pi^1, \pi^2} (\r)$ also maximises $V_{\pi^1, \pi^2} (\lambda_1 \r + \lambda_2 \mathbf{1})$ for $\lambda_1 \geq 0$ and $\lambda_2 \in \Reals$, and vice versa. This means that $\r$ is feasible if and only if every $\bar \r \in \minset(\r)$ is feasible. 
\end{proof}

\subsection{Proof of Theorem~\ref{proposition:ideal_environment}}

\begin{customthm}{2}
(A) If $\A_2$ responds optimally and (B) if for all $\mathcal{T}: \S \times A_2 \to \Delta(\S)$ there exists $\pi^1$ such that $\trans_{\pi^1} \equiv \mathcal{T}$, then there exists a policy $\pi^1$ with optimal response $\pi^2$ such that the feasible set of reward functions under $(\pi^1, \pi^2)$ is given by $\minset(\truereward)$, i.e.\ $\feasible ((\pi^1, \pi^2)) = \minset(\truereward)$. 
\end{customthm}
For the proof of Theorem~\ref{proposition:ideal_environment}, we will need the following technical lemma.
\begin{lemma}\label{appendix-lemma:half-spaces characterization}
Any (two-dimensional) plane $\mathcal{R} \subseteq \Reals^{N}$ can be uniquely characterized by the intersection of $N-1$ many half-spaces $H_i = \{ x \in \Reals^{N} \colon \varphi_i^\top x \geq 0\}$, where $\varphi_1, \dots, \varphi_{N-1} \in \Reals^N$ are vectors orthogonal to $\mathcal{R}$. 
%Similarly, any line $\mathcal{C}$ in $\Reals^{N}$ can be uniquely characterized by the intersection of $N$ half-spaces $H_i = \{ x \in \Reals^{N} \colon \varphi_i^\top x \geq 0\}$, where $\varphi_1, \dots, \varphi_{N} \in \Reals^N$ are vectors orthogonal to $\mathcal{C}$.
\end{lemma}
\begin{proof}[Proof of Lemma \ref{appendix-lemma:half-spaces characterization}]
W.lo.g.\ let $\mathcal{R}$ be some plane in $\Reals^N$ through the origin. 
Let the vectors $v_1$ and $v_2$ denote an orthogonal basis of $\mathcal{R}$, i.e.\  $\mathcal{R} = \{ \lambda_1 v_1 + \lambda_2 v_2 \colon \lambda_1, \lambda_2 \in \Reals \}$ and $v_1^\top v_2 = 0$. We can then find vectors $\varphi_1, \dots, \varphi_{N-2}$ such that $\{\varphi_1, \dots, \varphi_{N-2}, v_1, v_2\}$ forms an orthogonal basis of $\Reals^N$. 
In particular, we then have $\varphi_i^\top x = 0$ for all $x \in \mathcal{R}$ and $i \in [N-2]$. 
Moreover, we define the vector $$\varphi_{N-1} = - (\varphi_1 + \dots + \varphi_{N-2})$$ and note that $\varphi_{N-1}$ is orthogonal to $\mathcal{R}$ as well. 
Let the half-spaces induced by vectors $\varphi_1, \dots, \varphi_{N-1}$ be given by $H_i = \{ x \in \Reals^N \colon \varphi_i^\top x \geq 0\}$ for $i \in [N-1]$. We now show that $H_1 \cap \dots \cap H_{N-1} = \mathcal{R}$. 

We begin by verifying that $H_1 \cap \dots \cap H_{N-1} \subseteq \mathcal{R}$. 
Suppose this is not true and there exists a vector $w \notin \mathcal{R}$ such that $\varphi_i^\top w \geq 0$ for all $i \in [N-1]$, i.e.\ $w \in H_1 \cap \dots \cap H_{N-1}$. 
Then, we must have $\varphi_j^\top w > 0$ for some $j \in [N-2]$ as the orthogonal complement of $\spn(\varphi_1, \dots, \varphi_{N-2})$ is given by $\mathcal{R}$ and we assumed $w \notin \mathcal{R}$. 
By definition of $\varphi_{N-1}$, we have $\varphi_1 + \dots + \varphi_{N-1} = 0$ and thus, $(\varphi_1 + \dots + \varphi_{N-1})^\top w = 0$. 
However, it also holds that $$\varphi_1^\top w + \dots + \varphi_{N-1}^\top w > 0,$$ since $\varphi_i^\top w \geq 0$ for $i \in [N-1]$ and $\varphi_j^\top w > 0$ for some $j \in [N-2]$. Thus, such $w$ cannot exist and we have shown that $H_1 \cap \dots \cap H_{N-1} \subseteq \mathcal{R}$. 
Finally, the relation $\mathcal{R} \subseteq H_1 \cap \dots \cap H_{N-1}$ also holds as $\varphi_1, \dots, \varphi_{N-1}$ are chosen orthogonal to $\mathcal{R}$ and thus, $\varphi_i^\top  x = 0$ for all $i \in [N-1]$ and $x \in \mathcal{R}$. 

Note that we can analogously prove that any line $\mathcal{C} = \{ \lambda v : \lambda \in \Reals\}$ in $\Reals^N$ can be uniquely characterised by $N$ half-spaces. In this case, we can find an orthogonal basis $\{\varphi_1, \dots, \varphi_{N-1}, v\}$ and define $\varphi_N = -(\varphi_1+ \dots + \varphi_{N-1})$. The remainder of the proof then follows the same line of argument as before.
\end{proof}

\begin{proof}[Proof of Theorem~\ref{proposition:ideal_environment}]
Let $N = |\S|$. We will now show that under the assumptions of Theorem~\ref{proposition:ideal_environment}, there exists a policy $\pi^1$ with optimal response $\pi^2$ so that only positive affine transformations of $\truereward$ are feasible under observation $(\pi^1, \pi^2)$, i.e.\ $\mathcal{R}((\pi^1, \pi^2) = \minset(\truereward)$.

First we observe that we can w.l.o.g.\ assume only two actions for $\A_2$, i.e.\ $|A_2|=2$. To see this suppose that $|A_2|> 2$ and consider an action space $A_2^\prime \subset A_2$ with $|A_2^\prime| \geq 2$ and transition kernel $\trans^\prime_{\pi^1}: \S \times A_2^\prime \to \Delta(\S)$ defined as $\trans^\prime_{\pi^1} (\cdot \mid s, b) = \trans_{\pi^1} (\cdot \mid s, b)$ for $b \in A_2^\prime$.
If $\pi^2(s) \in A_2^\prime$ for all $s \in \S$, then the feasible set under action space $A_2$ is subset of the feasible set under action space $A_2^\prime$.
Thus, we can assume w.l.o.g. that $A_2 = \{b_1, b_2\}$. 
From hereon out, we assume that the true reward function $\truereward$ is non-constant. The special case of a constant true reward function is addressed at the end. 

We first construct an orthogonal basis $\{\varphi_1, \dots, \varphi_N\}$ such that the corresponding half-spaces characterise $\minset(\truereward$
%\footnote{Note, that to construct this basis, we need knowledge of $\truereward$. Thus, as we pointed out in the main paper, while the existence of an ideal environment $\trans_{\pi^1}$ that determines $\minset(\truereward)$ is guaranteed, computing such a policy $\pi_1$ is unrealistic.} 
and then show that there exists $\pi_1$ such that  $$(\trans_{\pi^1, b_1} - \trans_{\pi^1, b_2}) (I- \gamma \trans_{\pi^1, b_1})^{-1} = (\varphi_1, \dots, \varphi_N)^\top. $$
For non-constant $\truereward$ we have that $\mathcal{R} \triangleq \spn(\truereward, \mathbf{1})$ describes a plane in $\Reals^N$ and $\minset(\truereward) \subset \mathcal{R}$. By Lemma \ref{appendix-lemma:half-spaces characterization}, there exist vectors $\varphi_1, \dots, \varphi_{N-1} \in \Reals^N$ such that $\varphi_i^\top x = 0$ for all $x \in \mathcal{R}$ and $H_1 \cap \dots \cap H_{N-1} = \mathcal{R}$ with $H_i = \{ x \in \Reals^N \colon \varphi_i^\top x \geq 0\}$. 
In particular, it holds that $\varphi_i^\top \mathbf{1} = 0$, i.e.\ $\lVert \varphi_i \rVert_1 = 0$ for all $i \in [N-1]$.

Now, let us consider the orthogonal projection of $\truereward$ given by $\truereward = \alpha \mathbf{1} + w$ for $\alpha \in \Reals$ and $w \in \Reals^N$ with $w^\top \mathbf{1} = 0$. 
It follows that $w^\top \truereward = w^\top (\alpha \mathbf{1} + w) = w^\top w > 0$, since $\truereward$ is non-constant and thus, $w \neq \mathbf{0}$. 
Let us define $\varphi_N = \eta w$ for some scalar $\eta > 0$. Then, we have $\varphi_N^\top x \geq 0$ for all $x \in \{\lambda_1 \truereward + \lambda_2 \mathbf{1} \colon \lambda \geq 0, \lambda_2 \in \Reals\}$, since $w^\top \truereward > 0$ and $w^\top \mathbf{1} = 0$. Similarly, we have $\varphi_N^\top \hat x < 0$ for all $ \hat x \in \{ \lambda_1 \truereward + \lambda_2 \mathbf{1} \colon \lambda_1 < 0 , \lambda_2 \in \Reals \}$. It then follows that 
$$H_1 \cap \dots \cap H_N = \mathcal{R} \cap H_N = \minset(\truereward),$$ where $H_N = \{ x \in \Reals^N \colon \varphi_N^\top x \geq 0\}$. 
Note that every $\varphi_i$ with $i \in [N]$ satisfies $\lVert \varphi_i \rVert_1 = 0$ and that the half-spaces $H_i$ are invariant under positive linear transformation of $\varphi_i$. We can therefore assume that $\varphi_1, \dots, \varphi_N$ take values in $[\frac{1}{N} - 1, \frac{1}{N}]$. 
We denote with $\Phi = (\varphi_1, \dots, \varphi_N)^\top$ the matrix with rows $\varphi_1, \dots, \varphi_N$. 

% \kbcomment{Suppose that we have a large action space. It is enough for us to consider two actions. We can always reduce the large action space case to the small action space. So, let $A_1 = b_1, b_2$.}
% W.l.o.g.\ let us assume $A_2 = \{b_1, b_2\}$. This assumption is supported by the fact that for any action spaces $A_2^\prime \subset A_2$ and transition kernels $\trans^\prime_{\pi^1}: \S \times A_2^\prime \to \Delta(\S)$ and $\trans_{\pi^1}: \S \times A_2 \to \Delta(\S)$ with $\trans^\prime_{\pi^1} (\cdot \mid s, b) = \trans_{\pi^1} (\cdot \mid s, b)$ for $b \in A_2^\prime$, if $\pi^2(s) \in A_2^\prime$ for all $s \in \S$, the feasible set under action space $A_2$ is subset of the feasible set under action space $A_2^\prime$ \kbcomment{Clarify this.}. 
Recall that $A_2 = \{b_1, b_2\}$. We will now show that there exists a policy $\pi_1$ such that 
\begin{align*}
    (\trans_{\pi^1, b_1} - \trans_{\pi^1, b_2}) (I- \gamma \trans_{\pi^1, b_1})^{-1} = \Phi.
    %\begin{pmatrix} \varphi_1 \\ \vdots \\ \varphi_N \end{pmatrix}.
\end{align*}
By assumption, there exists a $\pi^1$ such that $\trans_{\pi^1, b_1} \equiv B_1$ and $\trans_{\pi^1, b_2} \equiv B_2$ for any two stochastic matrices $B_1$ and $B_2$. We set $\trans_{\pi^1, b_1}(s' \mid s) = \frac{1}{N}$ for all $s, s' \in \S$, which yields
\begin{align}\label{equation:phi_identity}
    \Phi (I-\gamma \trans_{\pi^1, b_1}) = \Phi - \gamma \Phi \trans_{\pi^1, b_1}  = \Phi,
\end{align}
since $\lVert \varphi_i\rVert_1 = 0$ for all $i \in [N]$ and $\trans_{\pi^1, b_1}$ is a constant matrix. Now, set $\trans_{\pi^1, b_2} \equiv \trans_{\pi^1, b_1} - \Phi$ and note that since $\lVert \varphi_i\rVert_1 = 0$ for all $i \in [N]$, the matrix $\trans_{\pi^1, b_2}$ is indeed stochastic. It then follows that  
\begin{align*}
    (\trans_{\pi^1, b_1} - \trans_{\pi^1, b_2})  (I- \gamma \trans_{\pi^1, b_1})^{-1} = \Phi (I- \gamma \trans_{\pi^1, b_1})^{-1} = \Phi, 
\end{align*}
by equation \eqref{equation:phi_identity}. Note that this means that indeed action $b_1$ is the optimal response to policy $\pi^1$ as $\Phi \truereward \succeq 0$ by construction of $\Phi$.\footnote{This can, for instance, be verified using Theorem~\ref{lemma:feasible_set_characterization}.} %Note that action $b_1$ is indeed the optimal response by agent $\A_2$ and $b_2$ suboptimal as 
%Note that $\Phi \truereward \geq 0$. 
Therefore, from Theorem \ref{lemma:feasible_set_characterization} it follows that any feasible reward function $\r$ must satisfy 
\begin{align*}
(\trans_{\pi^1,b_1} - \trans_{\pi^1, b_2})(I-\gamma \trans_{\pi^1, b_1})^{-1} \r = \Phi \r \succeq 0,
\end{align*}
i.e.\ $\varphi_i^\top \r \geq 0$ for all $i \in [N]$. Hence, any feasible reward function must be in $H_1 \cap \dots \cap H_N$ and thus element in $\minset(\truereward)$. So, we have shown that the feasible set of reward functions under $\pi^1$ with response $\pi^2 \equiv b_1$ is given by $\minset(\truereward)$.

In the special case of the constant reward function $\truereward$, we have that the set $\minset(\truereward) = \{ \lambda \mathbf{1} : \lambda \in \Reals\}$ becomes not a plane, but a line in $\Reals^{N}$. The proof for this case
% of Theorem~\ref{proposition:ideal_environment} 
then progresses similarly to the proof above with the difference that we describe $\minset(\truereward)$ by $N$ many half-spaces and that there is no need to consider the orthogonal projection of $\truereward$ as done before.

%Note that $\spn(\truereward, \mathbf{1})$ describes a plane in $\Reals^N$ and by Lemma \ref{appendix-lemma:half-spaces characterization} there exist vectors $\varphi_1, \dots, \varphi_{N-1}$   
\end{proof}

% There has been recent interest in the identifiability of reward functions. 

\subsection{Proof of Corollary~\ref{theorem:verifying_r_transform_of_true_reward}}
\begin{customcor}{\ref{theorem:verifying_r_transform_of_true_reward}}
Under Assumptions (A) and (B) of Theorem~\ref{proposition:ideal_environment}, the learner can verify in any episode whether a reward function $\r$ is a positive affine transformation of the actual and unknown reward function $\truereward$.
\end{customcor}
\begin{proof}
Recall that it follows from Lemma~\ref{lemma:minimal feasible set} that $\minset (\truereward) \subseteq \feasible((\pi^1, \pi^2))$ for any policy $\pi^1$ with optimal response $\pi^2$. In other words, the positive affine transformations of the unknown reward function $\truereward$ are always feasible as $\truereward$ is always feasible. 
Now, let $\r \in \Reals^{|\S|}$ be some reward function and suppose that $\A_1$ plays the ``ideal'' policy $\pi^1$ with respect to $\r$ as it is constructed in the proof of Theorem~\ref{proposition:ideal_environment}. Let $\pi^2$ be an optimal response to $\pi^1$. 
It follows from the combination of Lemma~\ref{lemma:minimal feasible set} and Theorem~\ref{proposition:ideal_environment} that $\feasible((\pi^1, \pi^2)) = \minset(\r)$ if and only if $\r \in \minset(\truereward)$. Now, using linear programming, we can check whether $\feasible((\pi^1, \pi^2)) = \minset(\r)$ holds true. If $\feasible((\pi^1, \pi^2)) = \minset(\r)$, we know that $\r$ must be a positive affine transformation of $\truereward$. On the other hand, if we observe $\feasible((\pi^1, \pi^2)) \neq \minset(\r)$, then $\r$ cannot be element in $\Aff(\truereward)$. 
% thereby telling us whether $\r$ is a positive affine transformation of the actual reward function $\truereward$ or not. 

\end{proof}

\iffalse
\begin{remark}
In Section \ref{subsection:bayesian_IRL}, we claim that for any $\beta \in \Reals$ and $\r \in \Reals^{| \S|}$, there always exists $\beta^\prime \in \Reals$ and $\r^\prime \in \Delta (\S)$ with $\beta^\prime Q(\r^\prime) \equiv \beta Q(\r)$ to support our assumption that the true reward function is in the simplex whenever we also perform inference over $\beta$. This follows directly from the fact that the $Q$-values are linear in $\r$ and thus, if $\r$ is the true reward function, then with $\beta^\prime = \beta \sum_{s \in \S} \r(s)$ and $\r^\prime(s)  = \r(s) / \beta^\prime$ we have $\beta^\prime Q(\r^\prime) = \beta Q(\r)$.% (\textcolor{blue}{right?}). 
\end{remark}
\fi

\subsection{Proof of Lemma~\ref{lemma:optimal_joint_policy_optimal}}

\begin{customlemma}{2}
Let $(\bar \pi^1, \bar\pi^2)$ % \in \argmax_{\pi^1, \pi^2} V_{\pi^1, \pi^2}$ 
be an optimal joint policy. If $\A_2$ responds optimally to the commitment of $\A_1$, then $V_{\bar \pi^1, \pi^2(\bar \pi^1)} = V_{\bar \pi^1, \bar \pi^2}$. In particular, this entails that $\max_{\pi^1} V_{\pi^1, \pi^2(\pi^1)} = \max_{\pi^1, \pi^2} V_{\pi^1, \pi^2}$.
\end{customlemma}

\begin{proof}[Proof of Lemma \ref{lemma:optimal_joint_policy_optimal}]
Let $(\bar \pi^1, \bar \pi^2) \in \argmax_{\pi^1, \pi^2} V_{\pi^1, \pi^2}$. Suppose $\A_1$ commits to $\bar \pi^1$. Then, $\A_2$ responds with $\pi^2(\bar \pi^1)$ such that $V_{\hat\pi^1, \pi^2(\bar \pi^1)} \succeq V_{\bar \pi^1, \pi^2}$ for all $\pi^2$ by optimality of $\A_2$. Now, since $V_{\bar \pi^1, \bar \pi^2} \succeq \max_{\pi^1} V_{\pi^1, \pi^2(\pi^1)}$ always, we also have  
\begin{align*}
    \max_{\pi^1}V_{\pi^1, \pi^2(\pi^1)}&  \succeq V_{\bar \pi^1, \pi^2(\bar \pi^1)} \succeq V_{\bar \pi^1, \bar \pi^2} \succeq \max_{\pi^1} V_{\pi^1, \pi^2(\pi^1)}.
\end{align*}
Thus, $\max_{\pi^1} V_{\pi^1, \pi^2(\pi^1)} = V_{\bar \pi^1, \bar \pi^2} = \max_{\pi^1, \pi^2} V_{\pi^1, \pi^2}$. In other words, Lemma~\ref{lemma:optimal_joint_policy_optimal} states that the optimal joint policy yields an optimal commitment strategy for $\A_1$ when $\A_2$ responds optimally.
\end{proof}

\subsection{Proof of Proposition~\ref{lemma:algorithm_converges}}
\setcounter{proposition}{0}
\begin{proposition}
Suppose that for any non-constant reward function $\r \in \Delta(\S)$ it holds that if an optimal joint policy $(\pi^1, \pi^2)$ under $\r$ is suboptimal under $\truereward$, then in return there exists an optimal response $\pi^2(\pi^1)$ under $\truereward$ that is suboptimal under $\r$. Moreover, assume that $\A_2$ responds optimally and breaks ties between equally good policies uniformly at random. Then, the average regret suffered by Algorithm~\ref{algorithm:full_info_opt_behaviour} converges to zero almost surely. 
\end{proposition}

For the proof of Proposition~\ref{lemma:algorithm_converges}, we will need the following sets: Let $\optPi(\r)$ denote the set of optimal joint policies under reward function $\r$, i.e.\ the set of optimal joint policies in the MDP $(\S, A_1, A_2, \trans, \r, \gamma)$. 
%If we consider a single-agent MDP $(\S, A, \trans, \truereward, \gamma)$, the set $\optPi(\r)$ will accordingly denote the set of optimal policies in $(\S, A, \trans, \r, \gamma)$. 
Further, we denote the set of optimal responses under policy $\pi^1$ and reward function $\r$ by $\optPi_2(\r, \pi^1)$. 
A key object of interest is the following set of reward functions.
Let $\mathcal{O}$ be the set of reward functions in $\Delta(\S)$ that always induce an optimal joint policy, i.e.\ $$\mathcal{O} = \{ \r \in \Delta(\S) : \optPi (\r) \subseteq \optPi (\truereward) \}. $$ %\mgcomment{Do we really want $\optPi (\r) = \optPi (\truereward)$? Is it always the case that $\optPi (\r) = \optPi (\truereward) $ or $\optPi (\r) \cap \optPi (\truereward) = \emptyset$? or could it be that $\optPi (\r) \subset \optPi (\truereward) $?}
Note that by Lemma~\ref{lemma:optimal_joint_policy_optimal} any optimal joint policy yields an optimal commitment strategy for agent $\A_1$, i.e.\ any $\r \in \mathcal{O}$ induces an optimal commitment strategy. We can easily check that $\mathcal{O}$ is a convex set.
% We will need the following results for the proof of Proposition~\ref{lemma:algorithm_converges}.

%Consider a standard single-agent MDP $(\S, A, \trans, \truereward, \gamma)$. Let $\mathcal{O}$ be the set of reward functions in $\Delta(\S)$ that always induce an (joint) optimal policy, i.e.\ $\mathcal{O} = \{ \r \in \Delta(\S) : \optPi (\r) = \optPi (\truereward)$. 
\begin{lemma}\label{aux_lemma:O_is_convex}
The set $\mathcal{O}$ is convex. 
\end{lemma}
\begin{proof}[Proof of Lemma~\ref{aux_lemma:O_is_convex}]
%W.l.o.g.\ we consider a single-agent MDP $(\S, A, \trans, \truereward, \gamma)$. Equivalently, we could model the two-agent MDP as a single-agent MDP with centralised controller, i.e.\ $(\S, A_1 \times A_2, \trans, \truereward, \gamma)$. 
Let $\r_1, \r_2 \in \mathcal{O}$. We show that $\lambda \r_1 + (1-\lambda) \r_2 \in \mathcal{O}$ for any $\lambda \in [0,1]$. Recall that the value function $V_\pi (\r) = (I-\gamma \trans_{\pi})^{-1} \r$ is linear in $\r$ and we therefore have $V_\pi (\lambda \r_1 + (1-\lambda) \r_2) = \lambda V_{\pi}(\r_1) + (1-\lambda) V_{\pi} (\r_2)$.
In a first step, we prove $\optPi(\lambda \r_1 + (1-\lambda) \r_2) \subseteq \optPi (\truereward)$.  Let $\pi \in \optPi (\lambda \r_1 + (1-\lambda) \r_2)$. Then, for all policies $\nu$ it must hold that 
\begin{align}\label{aux_equation:O_convex}
    V_{\pi} (\lambda \r_1 + (1-\lambda) \r_2) \succeq V_{\nu} (\lambda \r_1 + (1-\lambda) \r_2),
\end{align} where $\succ$ denotes element-wise inequality. Now, suppose that $\pi \notin \optPi (\truereward)$. It follows that $V_{\pi} (\r_1) \preceq V_{\nu} (\r_1)$ and $V_{\pi} (\r_2) \preceq V_{\nu} (\r_2)$ for some $\nu \in \optPi(\truereward) = \optPi(\r_1) = \optPi(\r_2)$ with strict inequality for at least one $s \in \S$. This contradicts equation~\eqref{aux_equation:O_convex} and it follows that $\optPi(\lambda \r_1 + (1-\lambda) \r_2) \subseteq \optPi (\truereward)$. We will now verify the relation $\optPi (\truereward) \subseteq \optPi(\lambda \r_1 + (1-\lambda) \r_2)$. For any $\pi \in \optPi(\truereward)$, we have $V_{\pi} (\r_1) \succeq V_{\nu} (\r_1)$ and $V_{\pi} (\r_2) \succeq V_{\nu} (\r_2)$ for all policies $\nu$. It then directly follows that $\pi \in \optPi(\lambda \r_1 + (1-\lambda) \r_2)$ and thus, $\optPi (\truereward) \subseteq \optPi(\lambda \r_1 + (1-\lambda) \r_2)$, i.e.\ $\lambda \r_1 + (1-\lambda) \r_2 \in \mathcal{O}$.
\end{proof}
Interestingly, Lemma~\ref{aux_lemma:O_is_convex} implies that the set of reward functions that induce an optimal commitment strategy is a connected set. %Therefore, there are no solitary reward functions that purely by chance induce an optimal policy.  
We will now prove Proposition~\ref{lemma:algorithm_converges}.

\begin{proof}[Proof of Proposition~\ref{lemma:algorithm_converges}]
As Algorithm~\ref{algorithm:full_info_opt_behaviour} only considers reward functions in the simplex $\Delta(\S)$, we will simply write $\feasible_t$ instead of $\feasible_t \cap \Delta(\S)$ for notational convenience. 
% Moreover, recall the assumption that if there are multiple optimal responses to some policy $\pi^1$, agent $\A_2$ will select one of these uniformly at random. 

In episode $t$, Algorithm~\ref{algorithm:full_info_opt_behaviour} chooses a vertex of the set of feasible solutions of the linear program, i.e.\ a reward function $\r_t \in \feasible_t$. 
Note that by construction of Algorithm~\ref{algorithm:full_info_opt_behaviour} we never select the constant reward function in $\Delta(\S)$. 
% We shall remind ourselves in the following that the special case of the constant reward function will thus not affect Algorithm~\ref{algorithm:full_info_opt_behaviour}. 
For any $\r_t \in \feasible_t$ obtained from the LP \eqref{equation:linear_program} with uniformly random objective function $c$ there are two possible cases: $\r_t \in \mathcal{O}$ or $\r_t \notin \mathcal{O}$. If $\r_t \in \mathcal{O}$, then $\r_t$ induces an optimal joint policy, i.e.\ an optimal commitment strategy by Lemma~\ref{lemma:optimal_joint_policy_optimal}.
Accordingly, Algorithm~\ref{algorithm:full_info_opt_behaviour} commits to an optimal commitment strategy and thus suffers zero regret in episode $t+1$. We want to highlight that the proof does not require that the objective function in Algorithm~\ref{algorithm:full_info_opt_behaviour} is being chosen in a randomised fashion. However, randomising the choice of the objective improved exploration in our experiments.

In the following, we show that for the case of $\r_t \notin \mathcal{O}$, Algorithm~\ref{algorithm:full_info_opt_behaviour} strictly decreases the set of feasible reward functions with positive probability. In order to show this, we first construct a finite cover of $\Delta(\S)$.
% \paragraph{Finite Partitions of $\Delta(\S)$.}
% In a first preliminary step, we will construct finite partitions of the simplex $\Delta(\S)$. 
Let $\Pi_1$ and $\Pi_2$ denote the sets of deterministic policies for $\A_1$ and $\A_2$, respectively.\footnote{We assume here that $\A_2$ responds with deterministic policies in order to keep the proof as comprehensible as possible. However, this assumption can be dropped as we can still give a finite partition of $\Delta(\S)$ when $\A_2$ also responds with optimal stochastic policies.} Note that both $\Pi_1$ and $\Pi_2$ are finite as we assumed finite action spaces $A_1$ and $A_2$. Let $2^{\Pi_2}$ denote the power set of $\Pi_2$. For $\pi^1 \in \Pi_1$ and $\bar \Pi_2 \in 2^{\Pi_2}$, we define 
\begin{align*}
    B(\pi^1, \bar \Pi_2) = \{ \r \in \Delta(\S) : \bar \Pi_2 = \optPi_2(\r, \pi^1)\}.
\end{align*}
The set $B(\pi^1, \bar \Pi_2)$ thus describes the reward functions that make the policies in $\bar \Pi_2$ optimal in response to $\pi^1$. Indeed, for any fixed $\pi^1 \in \Pi_1$, the collection $\mathcal{B}(\pi^1) = \{ B(\pi^1, \bar \Pi_2) : \bar \Pi_2 \in 2^{\Pi_2}\}$ forms a finite partition of $\Delta(\S)$%\mgcomment{i think partition is not a good word here, because sets are not disjoint. more accurately it is a cover.}, i.e.\ 
\begin{align*}
    {\bigcup}_{\bar \Pi_2 \in 2^{\Pi_2}} B(\pi^1, \bar \Pi_2) = \Delta(\S),
\end{align*}
as for any $\r \in \Delta(\S)$ there always exists at least one deterministic optimal policy in the MDP $(\S, A_2, \trans_{\pi^1}, \r, \gamma)$ \citep{puterman2014markov}. In other words, for any $\pi^1 \in \Pi_1$, we partition $\Delta(\S)$ into sets that induce the same set of optimal responses to $\pi^1$. 
Naturally, due to $\mathcal{B}(\pi^1)$ being a {\em finite} partition of $\Delta(\S)$ for any $\pi^1$, the Lebesgue-measure for all but finitely many $B(\pi^1, \bar \Pi_2)$ must be larger than some constant $\varepsilon>0$.  

We now show that if $\r_t \notin \mathcal{O}$, then with positive probability the set of feasible solutions is decreased by at least~$\varepsilon$. 
If $\r_t \notin \mathcal{O}$, then Algorithm~\ref{algorithm:full_info_opt_behaviour} computes an optimal commitment strategy $\pi^1_{t+1}\in \optPi_1(\r_t)$ (by computing the optimal joint policy under $\r_t$, see Lemma~\ref{lemma:optimal_joint_policy_optimal}), which may be suboptimal under $\truereward$, i.e.\ $\pi^1_{t+1} \notin \optPi_1(\truereward)$. 

Now, if $\pi^1_{t+1}$ is suboptimal under $\truereward$, then by assumption\footnote{Note that if $\pi^1$ is a suboptimal commitment strategy, then the joint policy $(\pi^1, \pi^2)$ is suboptimal for any $\pi^2$.} there exists an optimal response $\pi^2_{t+1} \in \optPi_2(\truereward, \pi^1_{t+1})$ that is suboptimal under $\r_t$, i.e.\ $\pi^2_{t+1} \notin \optPi_2(\r_t, \pi^1_{t+1})$. 
Recall that by our assumption $\A_2$ selects its response uniformly at random from $\optPi_2(\truereward, \pi^1_{t+1})$. Since $\optPi_2(\truereward, \pi^1_{t+1})$ is finite, $\A_2$ will respond with $\pi^2_{t+1} \notin \optPi_2(\r_t, \pi^1_{t+1})$ with positive probability. 

In that case, after observing $\pi^2_{t+1}$ the reward function $\r_t$ cannot be feasible anymore, i.e.\ $\r_t \notin \feasible_{t+1}$. 
In addition, we then also have that $B(\pi^1_{t+1}, \optPi_2(\r_t, \pi^1_{t+1})) \cap \feasible_{t+1} = \emptyset$, as all reward functions in $B(\pi^1_{t+1}, \optPi_2(\r_t, \pi^1_{t+1}))$ induce the same optimal responses $\optPi_2(\r_t, \pi^1_{t+1})$ and $\pi^2_{t+1}$ is not in $\optPi_2(\r_t, \pi^1_{t+1})$. 
In other words, any $\r \in B(\pi^1_{t+1}, \optPi_2(\r_t, \pi^1_{t+1}))$ cannot satisfy the constraints of Corollary~\ref{corollary:feasible_set}.
%Thus, we have $\feasible_{t+1} \subseteq \feasible_t \setminus B(\pi^1, \bar \Pi_2)$ for $B(\pi^1, \bar \Pi_2)$ with $\r_t \in B(\pi^1, \bar \Pi_2)$.\mgcomment{we need to show that our LP formulation supports this! also, what is $\bar \Pi_2$???} 

As seen before, for all but finitely many $ \bar \Pi_2 \in 2^{\Pi_2}$ we have $\lambda(B(\pi^1, \bar \Pi_2)) > \varepsilon$, where $\lambda$ is the Lebesgue-measure. As a consequence, if $\r_t \notin \mathcal{O}$, then we have for all but finitely many cases that $\lambda (\feasible_{t+1}) \leq \lambda(\feasible \setminus B(\pi^1_{t+1}, \optPi_2(\r_t, \pi^1_{t+1})) \leq \lambda(\feasible_t) - \varepsilon$. 
%for $B(\pi^1, \bar \Pi_2)$ with $\r_t \in B(\pi^1, \bar \Pi_2)$. 

Therefore, every time when Algorithm~\ref{algorithm:full_info_opt_behaviour} chooses a reward function $\r_t \notin \mathcal{O}$\footnote{Recall that the special case of the constant reward function (which is not in $\mathcal{O}$) can be ignored.} inducing a suboptimal commitment strategy, (with positive probability) $\r_t$ will not be feasible anymore and (except for finitely many times) we reduce the size of the feasible set by at least the constant amount $\varepsilon$. As a result, the feasible set of reward function $\feasible_t$ will eventually become smaller than or equal to $\mathcal{O}$, i.e.\ $\feasible_t \subseteq \mathcal{O}$. Consequently, Algorithm~\ref{algorithm:full_info_opt_behaviour} will almost surely converge to choosing only reward function in $\mathcal{O}$ and will thus only play optimal commitment strategies.  
\end{proof}

\section{Proofs for Section~\ref{sec:suboptimal}}\label{appendix:proofs_suboptimal}

\subsection{Proof of Theorem~\ref{lemma:no_dominating_policy}}
\begin{customthm}{3}
If $\pi^2 (\actb \mid s) \propto f(Q_{\pi^1}^*(s, \actb))$ for any strictly increasing function $f:[0, \infty) \to [0, \infty)$, then a dominating commitment strategy for agent $\A_1$ may not exist.
\end{customthm}

\begin{proof}[Proof of Theorem \ref{lemma:no_dominating_policy}]
We provide a problem instance for which there exists no dominating policy for any strictly increasing function $f:[0, \infty) \to [0, \infty)$. Consider the two-agent MDP in Figure~\ref{figure:counterexample_mdp}. We omitted consecutive transitions in Figure~\ref{figure:counterexample_mdp}, but assume that states $s_1, s_3$, and $s_4$ lead to the same (terminal) state with probability one.

\colorlet{DarkGreen}{green!25!black!75}

\tikzset{node distance=0.5cm and 2.75cm, % Minimum distance between two nodes. Change if necessary.
every state/.style={ % Sets the properties for each state
semithick,
fill=gray!10
},
initial text={}, % No label on start arrow
double distance=5pt, % Adjust appearance of accept states
every edge/.style={ % Sets the properties for each transition
draw,
%->,>=stealth’, % Makes edges directed with bold arrowheads
auto,
semithick}}

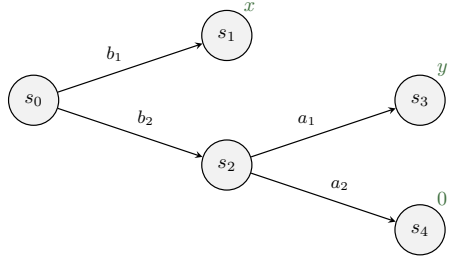
\begin{figure}
\centering 
\resizebox{.35\textwidth}{!}{
\begin{tikzpicture}[>=stealth]
\node[state] (s1) {$s_0$};
\node[state, above right=of s1, label={[label distance = -0.1cm]60:\textcolor{DarkGreen}{$x$}}] (s2) {$s_1$};
\node[state, below right=of s1] (s3) {$s_2$};
%\node[state, right=of s2, label=70:\textcolor{DarkGreen}{$x$}] (s4) {4};
\node[state, above right =of s3, label={[label distance = -0.1cm]60:\textcolor{DarkGreen}{$y$}}] (s4) {$s_3$};
\node[state, below right =of s3, label={[label distance = -0.1cm]60:\textcolor{DarkGreen}{$0$}}] (s5) {$s_4$};
% end state
% \node[state, right =of s4] (s6) {$s_6$};

\draw[->] (s1) edge node {\small$b_1$} (s2);
\draw[->] (s1) edge node {\small$b_2$} (s3);
%\draw[->] (s2) edge node {$a_1$, $a_2$} (s4);
\draw[->] (s3) edge node {\small$a_1$} (s4);
\draw[->] (s3) edge node {\small$a_2$} (s5);
% end state
% \draw[->] (s2) edge[bend left] node {} (s6);
% \draw[->] (s4) edge node {} (s6);
% \draw[->] (s5) edge[bend right] node {} (s6);
\end{tikzpicture}}
\caption{Counterexample. All transitions are deterministic. The action of $\A_2$ alone determines the transitions from state $s_0$ to states $s_1$ and $s_2$, whereas in state $s_2$ only the action of $\A_1$ affects transitions. The green $x$, $y$ and $0$ denote the rewards obtained in states $s_1$, $s_3$, and $s_4$, respectively. States $s_0$ and $s_2$ yield zero reward.}
\label{figure:counterexample_mdp}
\end{figure}

We will show that the strictly optimal policy when in state $s_0$ is strictly suboptimal when in state $s_2$ for specific choices of $x > 0$ and $y > 0$. For simplicity, we omit the discount factor $\gamma$ in the following. 

$\A_1$ only influences transitions in state $s_2$ and thus there are essentially only two deterministic policies for $\A_1$, namely $\pi^1$ with $\pi^1(s_2) = a_1$ and $\bar \pi^1$ with $\bar \pi^1(s_2) = a_2$. Since $y > 0$, action $a_1$ is optimal in state $s_2$ and so $\pi^1$ is the optimal policy in state $s_2$. We now show that there exists $x, y > 0$ such that $V_{\pi^1, \pi^2(\pi^1)} (s_0) < V_{\bar \pi^1, \pi^2(\bar \pi^1)} (s_0)$, i.e.\ $\bar \pi^1$ is strictly better than $\pi^1$ when in state $s_0$.

Omitting the discount factor, we have $Q_{\pi^1}^*(s_0, b_1) = x$ and $Q_{\pi^1}^*(s_0, b_2) = y$ as well as $Q_{\bar \pi^1}^*(s_0, b_1) = x$ and $Q_{\bar \pi^1}^*(s_1, b_2) = 0$. We therefore want to show that there exist $x, y > 0$ such that
\begin{align*}
    V_{\pi^1, \pi^2(\pi^1)} (s_1) &  = x \, \frac{f(x)}{f(x) + f(y)} + y \, \frac{f(y)}{f(x) + f(y)} \\
    & <  x \, \frac{f(x)}{f(x) + f(0)}  = V_{\bar \pi^1, \pi^2(\bar \pi^1)} (s_1).
\end{align*}
\vspace{1cm}

Suppose the contrary is true. Then, for all $x, y > 0$ it must hold that
\begin{align}
\label{equation:proof_boltzmann_1}
    \quad x \, \frac{f(x)}{f(x) + f(y)} + y \, \frac{f(y)}{f(x) + f(y)} & \geq  x \, \frac{f(x)}{f(x) + f(0)} \nonumber \\
    \Leftrightarrow \ x \, \Big( \frac{f(x)}{f(x) + f(0)} - \frac{f(x)}{f(x) + f(y)}  \Big) & \leq y \, \frac{f(y)}{f(x) + f(y)} \nonumber \\
    \Leftrightarrow \quad \qquad x f(x) \, \Big( \frac{f(x) + f(y)}{f(x) + f(0)} - 1\Big) & \leq y f(y) \nonumber \\
    \Leftrightarrow \qquad x f(x) \, \frac{f(y) - f(0)}{f(x) + f(0)} \leq y f(y).
\end{align}
\iffalse
Then, by rearranging terms, we have that for all $x, y > 0$
\begin{align}\label{equation:proof_boltzmann_1}
        x f(x) \, \frac{f(y) - f(0)}{f(x) + f(0)} \leq y f(y).
\end{align}
\fi
Note that $f(y) - f(0) > 0$, since $f$ is strictly increasing. Now, for any fixed $y> 0$, we have that $f(x) \frac{f(y) - f(0)}{f(x) + f(0)} \to 1$ as $x \to \infty$, and the expression is therefore bounded from below by some positive value for $x$ sufficiently large. Hence, for any fixed $y$ there exists an $x > 0$ such that \eqref{equation:proof_boltzmann_1} does not hold. This shows that in fact for any $y > 0$ there exists $x > 0$ such that $V_{\pi^1, \pi^2(\pi^1)} (s_0) < V_{\bar \pi^1, \pi^2(\bar \pi^1)}(s_0)$, whereas we have seen before that $V_{\pi^1, \pi^2(\pi^1)} (s_2) > V_{\bar \pi^1, \pi^2(\bar \pi^1)}(s_2)$. Hence, no dominating commitment strategy exists for the MDP depicted in Figure~\ref{figure:counterexample_mdp}.
\end{proof}

\subsection{Proof of Lemma~\ref{lemma:eps_greedy_dominating}}
\begin{customlemma}{3}
If $\A_2$ plays $\varepsilon$-greedy responses, a dominating commitment strategy for $\A_1$ may not exist. 
\end{customlemma}

\noindent We define an $\varepsilon$-greedy response to a policy $\pi^1$ as the policy
\begin{align*}
    \pi^2_\varepsilon (s, \pi^1) =
    \begin{cases}
    \pi^2_*(s, \pi^1) & w.p.\ 1-\varepsilon \\
    \mathcal{U}(A_2) & w.p.\ \varepsilon,
    \end{cases}
\end{align*}
where $\varepsilon \in [0, 1]$, $\pi^2_*(\pi^1)$ is an optimal response to $\pi^1$, and $\mathcal{U}(A_2)$ the uniform distribution over $A_2$. 

\colorlet{DarkGreen}{green!25!black!75}

\tikzset{node distance=0.75cm and 2.75cm, % Minimum distance between two nodes. Change if necessary.
every state/.style={ % Sets the properties for each state
semithick,
fill=gray!10},
initial text={}, % No label on start arrow
double distance=5pt, % Adjust appearance of accept states
every edge/.style={ % Sets the properties for each transition
draw,
%->,>=stealth’, % Makes edges directed with bold arrowheads
auto,
semithick}}

\begin{figure}[t]
\centering 

\resizebox{.45\textwidth}{!}{
\begin{tikzpicture}[>=stealth]
\node[state] (s0) {$s_0$};
\node[state, above right=of s0, label={[label distance = -0.1cm]60:\textcolor{DarkGreen}{$+1$}}] (s1) {$s_1$};
\node[state, below right=of s0] (s2) {$s_2$};
\node[state, above right =of s2, label={[label distance = -0.1cm]60:\textcolor{DarkGreen}{$+2$}}] (s3) {$s_3$};
% \node [label={[label distance=1cm]30:label}] {Node};
\node[state, right =of s2, label={[label distance = -0.2cm]60:\textcolor{DarkGreen}{$- \frac{2(2-\delta)(1-\varepsilon/2)}{\varepsilon}$}}] (s4) {$s_4$};
\node[state, below right =of s2, label={[label distance = -0.1cm]60:\textcolor{DarkGreen}{$0$}}] (s5) {$s_5$};
% end state
% \node[state, right =of s3] (s5) {$s_5$};

\draw[->] (s0) edge node {\small$b_1$} (s1);
\draw[->] (s0) edge node {\small$b_2$} (s2);
\draw[->] (s2) edge node {\small$(a_1, b_1)$} (s3);
\draw[->] (s2) edge node {\small$(a_1, b_2)$} (s4);
\draw[->] (s2) edge node {\small$a_2$} (s5);

% % end state
% \draw[->] (s1) edge[bend left] node {} (s5);
% \draw[->] (s3) edge node {} (s5);
% \draw[->] (s4) edge[bend right] node {} (s5);
\end{tikzpicture}}
\caption{Counterexample for $\varepsilon$-greedy responses. All transitions are deterministic. The actions from agent $\A_2$ alone determine the transitions from state $s_0$ to states $s_1$ and $s_2$. The green numbers denote the rewards obtained in the respective states. States $s_0$ and $s_2$ yield zero reward.}
\label{fig: counterexample-lem3}
\end{figure}
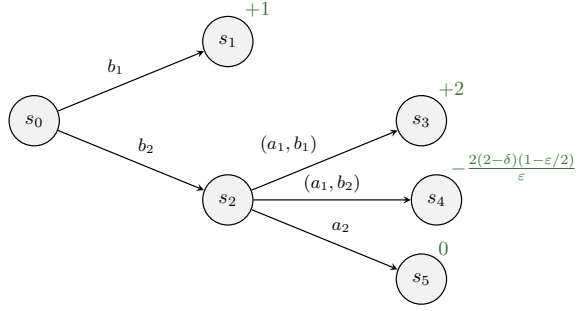

\begin{proof}[Proof of Lemma \ref{lemma:eps_greedy_dominating}]
We prove Lemma~\ref{lemma:eps_greedy_dominating} by means of the counterexample shown in Figure~\ref{fig: counterexample-lem3}. For convenience, we omit the discount factor here and assume that states $s_1$, $s_3$, $s_4$, and $s_5$ lead to some terminal state with probability one. There are two (deterministic) policies $\A_1$ can commit to: $\pi^1(s_2) = a_1$ and $\bar \pi^1(s_2) = a_2$.  

For notational convenience, we write $\smash{V_{a_1}(s) \defn V_{\pi^1, \pi^2_\varepsilon(\pi^1)} (s)}$ and $\smash{V_{a_2}(s) \defn V_{\bar \pi^1, \pi^2_\varepsilon(\bar \pi^1)} (s)}$. Note that if $\A_1$ commits to $\pi^1$, the optimal action for $\A_2$ in state $s_0$ is to play $b_2$ followed by $b_1$ in state $s_2$. Recall that $\A_2$ is assumed to play $\epsilon$-greedy, i.e.\ in any state, $\A_2$ plays the optimal response with probability $(1-\epsilon)$ and with probability $\epsilon$ selects an action uniformly at random. 
As a result, we have 
\begin{align*}
    V_{a_1}(s_2) & = 2(1-\varepsilon/2) - (2-\delta) (1-\varepsilon/2) = \delta(1-\varepsilon/2) > 0 \\ V_{a_1} (s_0) & = \delta (1-\varepsilon/2)^2 + \varepsilon/2 .
\end{align*}
On the other hand, if $\A_1$ commits to $\bar \pi^1$, it is optimal for $\A_2$ to play $b_1$ in state $s_0$, i.e.\ $V_{a_2}(s_0) = (1-\varepsilon/2)$. We observe that in state $s_2$, playing $a_1$ is optimal as $V_{a_1}(s_2) > V_{a_2} (s_2) = 0$. However, we also have 
    $V_{a_1} (s_0) - V_{a_2} (s_0) = \varepsilon + \delta(1-\varepsilon/2)^2 - 1.$
As we can choose $\delta$ arbitrarily close to $0$, we then have $V_{a_1}(s_0) < V_{a_2}(s_0)$ for some $\delta > 0$. Thus, $\pi^1$ is strictly optimal in state $s_2$, whereas $\bar \pi^1$ is strictly optimal in state $s_0$. Therefore, there exists no dominating commitment strategy for the MDP in Figure~\ref{fig: counterexample-lem3}.   

\end{proof}

\vspace{-.2cm}

\newpage 
\section{Approximate Algorithms for Cooperative Stackelberg Games with Suboptimal Followers}\label{appendix:approx_algorithms}

In this section, we first describe approximate value iteration algorithms for Boltzmann-rational policies as well as $\varepsilon$-greedy policies. We then evaluate both algorithms in the Maze-Maker and Random MDP environment for different levels of rationality (i.e.\ optimality) of agent $\A_2$.

\subsection{$\A_2$ responds with Boltzmann-rational policies}
Theorem~\ref{lemma:no_dominating_policy} states that no dominating commitment strategy may exist when agent $\A_2$ responds with Boltzmann-rational policies. In its essence, the approximate value iteration algorithm for Boltzmann-rational responses described in Algorithm~\ref{algorithm:value_iteration_boltzmann} acts as if a dominating commitment strategy does exist and could therefore converge to suboptimal solutions. However, it aims to account for the suboptimality of agent $\A_2$ and keeps track of two sets of value functions: one value function corresponding to what $\A_1$ believes to be the actual value given that $\A_2$ plays Boltzmann, and one value function that aims to approximate the belief of agent $\A_2$ about the value of the game.

\begin{algorithm}[H]
\caption{Approximate Value Iteration for Boltzmann-Rational Responses}
\label{algorithm:value_iteration_boltzmann}
\begin{algorithmic}[1]
\STATE \textbf{initialise} $V$ and $\hat V$
\STATE \textbf{repeat} until $V$ converges:
\FOR{$s \in \S$}
\FOR{$(a,b) \in A_1 \times A_2$}
\STATE $\hat Q(s,a, b) = \r(s) + \gamma \sum_{s'} \trans(s' | s, a, b) \hat V(s')$
\STATE $\pi^2(b\mid s, a) = \exp(\beta \hat Q(s, a, b) ) / Z$
\ENDFOR
\STATE $\pi^1(s) = \argmax_a \sum_{s'} \E_{b \sim \pi^2} [\trans(s' | s, a, b)] V(s')$ % \sum_b \pi^2(b | s, a) \trans (s' | s, a, b) V(s')$
\STATE $V(s) = \r(s) + \nolinebreak \gamma \sum_{s'} \E_{b \sim \pi^2} [\trans(s' | s, \pi^1(s), b)] V(s')$ % \sum_b \pi^2(b | s, \pi^1) \trans (s' | s, \pi^1, b) V(s')$ 
\STATE $\hat V (s) = \max_b \hat Q (s, \pi^1(s), b)$
\ENDFOR
\end{algorithmic}
\end{algorithm}

\subsection{{$\A_2$ responds with $\varepsilon$}-greedy policies}\label{appendix:greedy_policies}
The problem of planning with an agent that responds with $\varepsilon$-greedy policies is similar to the setting considered by \citet{dimitrakakis2017multi} in the sense that $\A_2$ plans with the original transition kernel $\trans$ (by computing an optimal response $\pi^2_*(\pi^1)$), whereas $\A_1$ plans (or should plan) with the ``correct'' transition kernel
\begin{align*}
    \trans_\varepsilon(\cdot \mid s, a, b) \equiv \varepsilon \trans (\cdot \mid s, a, \mathcal{U}(A_2)) + (1-\varepsilon) \trans (\cdot \mid s, a, b).
\end{align*} 
In particular, note that $\varepsilon \trans(s' \mid s, a, \mathcal{U}(A_2))$ is independent of the choice of $b$. 
Algorithm \ref{algorithm:value_iteration_greedy} approximately solves the planning problem. 
While Lemma~\ref{lemma:eps_greedy_dominating} states that a dominating commitment policy need not exist, Algorithm~\ref{algorithm:value_iteration_greedy} simply acts as if one exists. Similarly to Algorithm~\ref{algorithm:value_iteration_boltzmann}, the idea is to maintain two value functions, one representing the value from the perspective of $\A_1$ and the other the value from the perspective of $\A_2$.

\begin{algorithm}[H]
\caption{Approximate Value Iteration for $\varepsilon$-Greedy Responses}
\label{algorithm:value_iteration_greedy}
\begin{algorithmic}[1]
\STATE \textbf{initialise} $V$ and $\hat V$
\STATE \textbf{repeat} until $V$ converges:
\FOR{$s \in \S$}
\FOR{$a \in A_1$}
\STATE $\pi^2(s, a) = \argmax_b \sum_{s'} \trans(s' | s, a, b) \hat V(s')$
\ENDFOR
\STATE $\pi^1(s) = \argmax_a \sum_{s'} \E_{b \sim \pi^2} [ \trans_{\varepsilon} (s' | s, a, b)] V(s')$
\STATE $V(s) = \r(s) + \nolinebreak \gamma \sum_{s'} \E_{b \sim \pi^2} [ \trans_{\varepsilon} (s'| s, \pi^1(s), b)]V(s')$
\STATE $\hat V(s) = \r(s) + \gamma \sum_{s'} \E_{ b \sim \pi^2} [ \trans (s'| s, \pi^1(s), b)] \hat V(s') $
%\STATE $\pi^1(s) = \argmax_a \sum_{s'} \trans_\varepsilon (s' \mid s, a, \pi^2(s, a)) V(s')$
%\STATE $V(s) = \r(s)$\\ $\qquad\quad + \gamma \sum_{s'} \trans_\varepsilon (s' \mid s, \pi^1(s), \pi^2(s, \pi^1(s))) V(s')$
\ENDFOR
\end{algorithmic}
\end{algorithm}

\subsection{Evaluation of Algorithm~\ref{algorithm:value_iteration_boltzmann} and Algorithm~\ref{algorithm:value_iteration_greedy}}\label{appendix:approx_evaluation}
In this section, we empirically evaluate our approximate value iteration algorithms for Boltzmann-rational responses (Algorithm~\ref{algorithm:value_iteration_boltzmann}) and $\varepsilon$-greedy responses (Algorithm~\ref{algorithm:value_iteration_greedy}). 
%As there is (most likely) no efficient method to compute an optimal commitment strategy when $\A_2$ plays Boltzmann-rational or $\varepsilon$-greedy responses as there may not exist dominating commitment strategies, we are unable compare Algorithm~\ref{algorithm:value_iteration_boltzmann} and Algorithm~\ref{algorithm:value_iteration_greedy} against the optimal commitment strategy. 
We compare Algorithm~\ref{algorithm:value_iteration_boltzmann} and Algorithm~\ref{algorithm:value_iteration_greedy} in the Maze-Maker and Random MDP environment against committing $\A_1$'s part of the optimal joint policy. Note that by Lemma~\ref{lemma:optimal_joint_policy_optimal}, committing $\A_1$'s part of an optimal joint policy is optimal when $\A_2$ responds optimally.  

In both environments, we test the performance of our algorithms for different levels of rationality of $\A_2$. For the case of Boltzmann-rational responses (Figure~\ref{figure:appendix_boltz_approx}), we increase the inverse temperature of agent $\A_2$, which corresponds to the rationality (i.e.\ optimality) of $\A_2$. We see in Figure~\ref{figure:appendix_boltz_approx} that Algorithm~\ref{algorithm:value_iteration_boltzmann} consistently outperforms playing $\A_1's$ part of the optimal joint policy. In particular, the more suboptimal $\A_2$ is playing (lower values of $\beta$), the larger the advantage of Algorithm~\ref{algorithm:value_iteration_boltzmann} is compared to playing $\A_1$'s part of the optimal joint policy. If $\A_2$ responds almost optimally ($\beta = 20$), the performance of both approaches is almost identical as expected.

For the case of $\varepsilon$-greedy responses (Figure~\ref{figure:appendix_greedy_approx}), we increase the rationality of $\A_2$ by decreasing the probability $\varepsilon$ of random actions. Figure~\ref{figure:appendix_greedy_approx} shows that Algorithm~\ref{algorithm:value_iteration_greedy} outperforms playing the optimal joint policy for all values of $\varepsilon$ in both environments. In particular, for $\varepsilon = 0$ agent $\A_2$ responds optimally and both approaches play an optimal commitment strategy.

\begin{figure}[H]
    \subfigure[Maze-Maker]{%
      \includegraphics[width=0.45\linewidth]{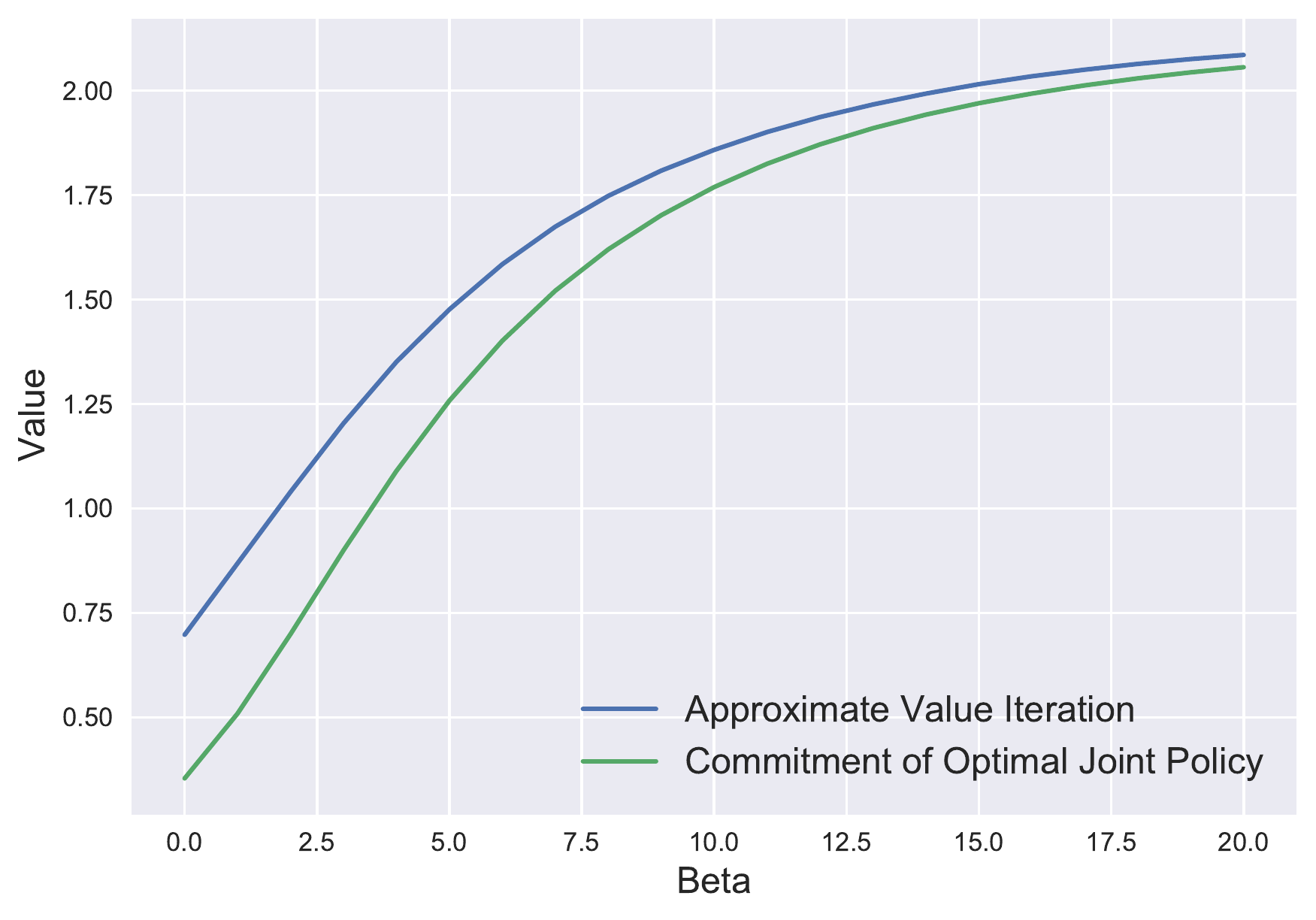}
    }
    \hspace{1cm}
    \subfigure[Random MDPs]{%
      \includegraphics[width=0.45\linewidth]{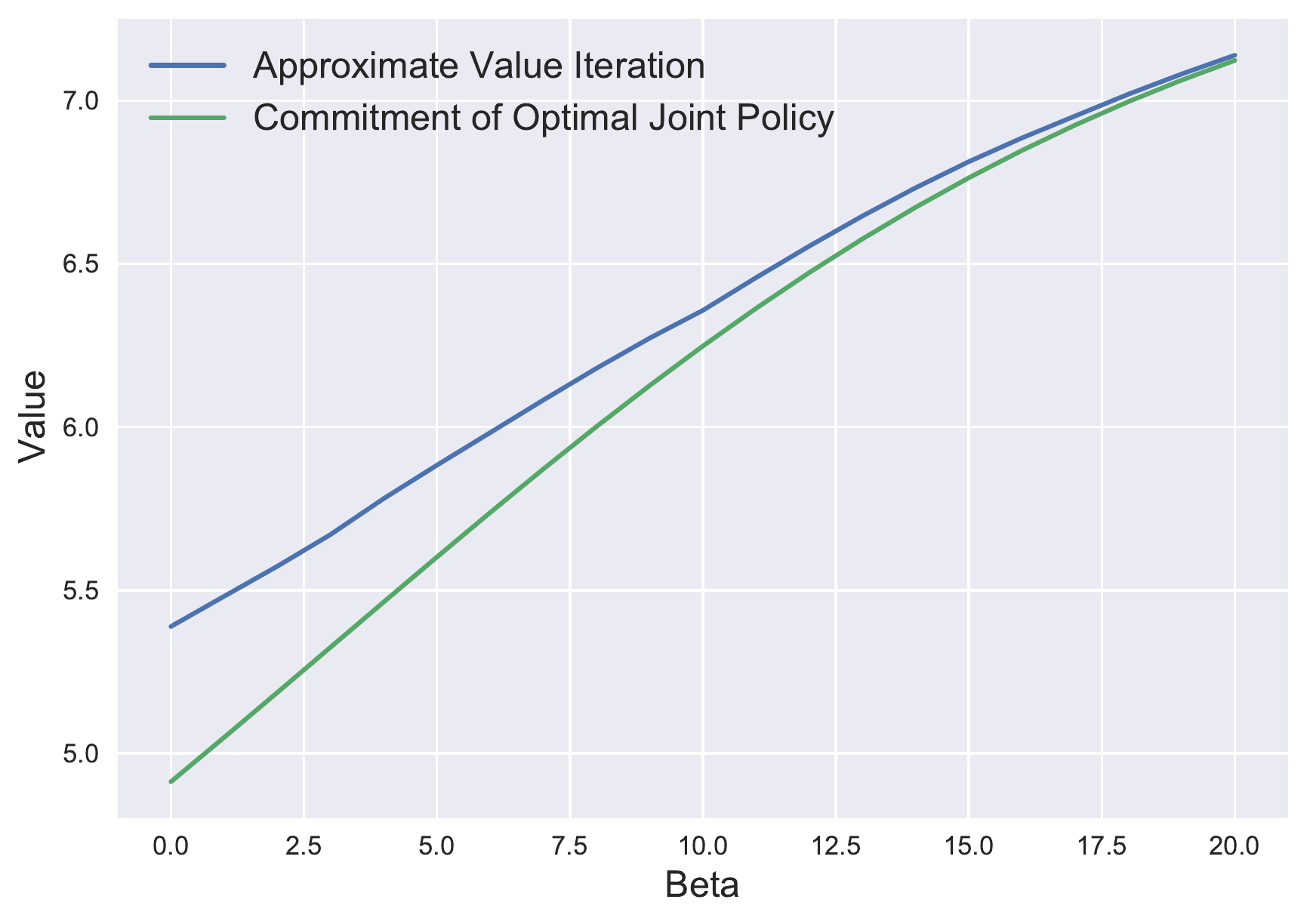}
    }
    \caption{Evaluation of Approximate Value Iteration for Boltzmann-Rational Responses (Algorithm~\ref{algorithm:value_iteration_boltzmann}) in the Maze-Maker and Random MDP environment for increasing values of $\beta$. The green line describes the return of playing $\A_1$'s part of an optimal joint policy.}
    \label{figure:appendix_boltz_approx}
\end{figure}

\begin{figure}[H]
    \subfigure[Maze-Maker]{%
      \includegraphics[width=0.45\linewidth]{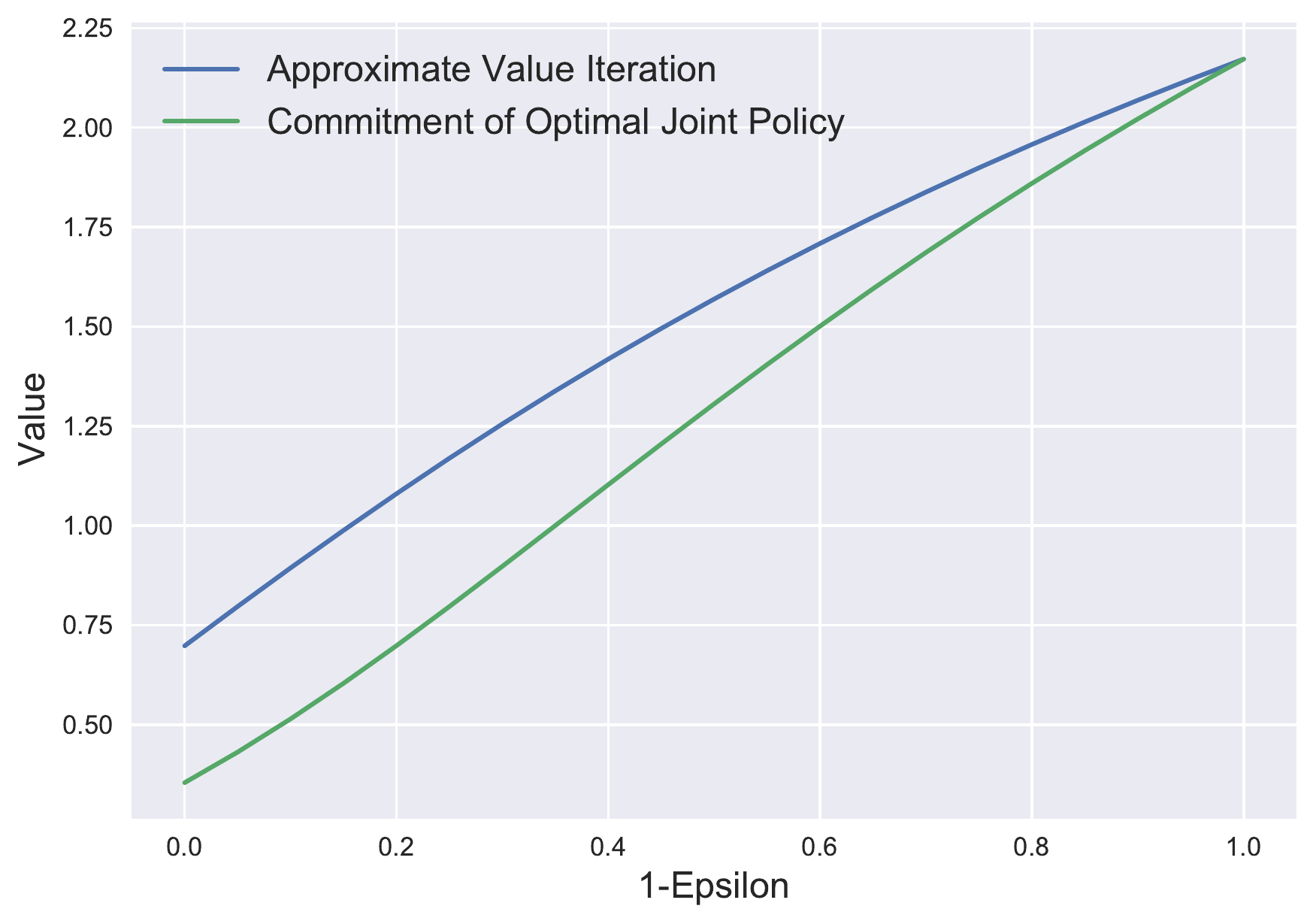}
    }
    \hspace{1cm}
    \subfigure[Random MDPs]{%
      \includegraphics[width=0.45\linewidth]{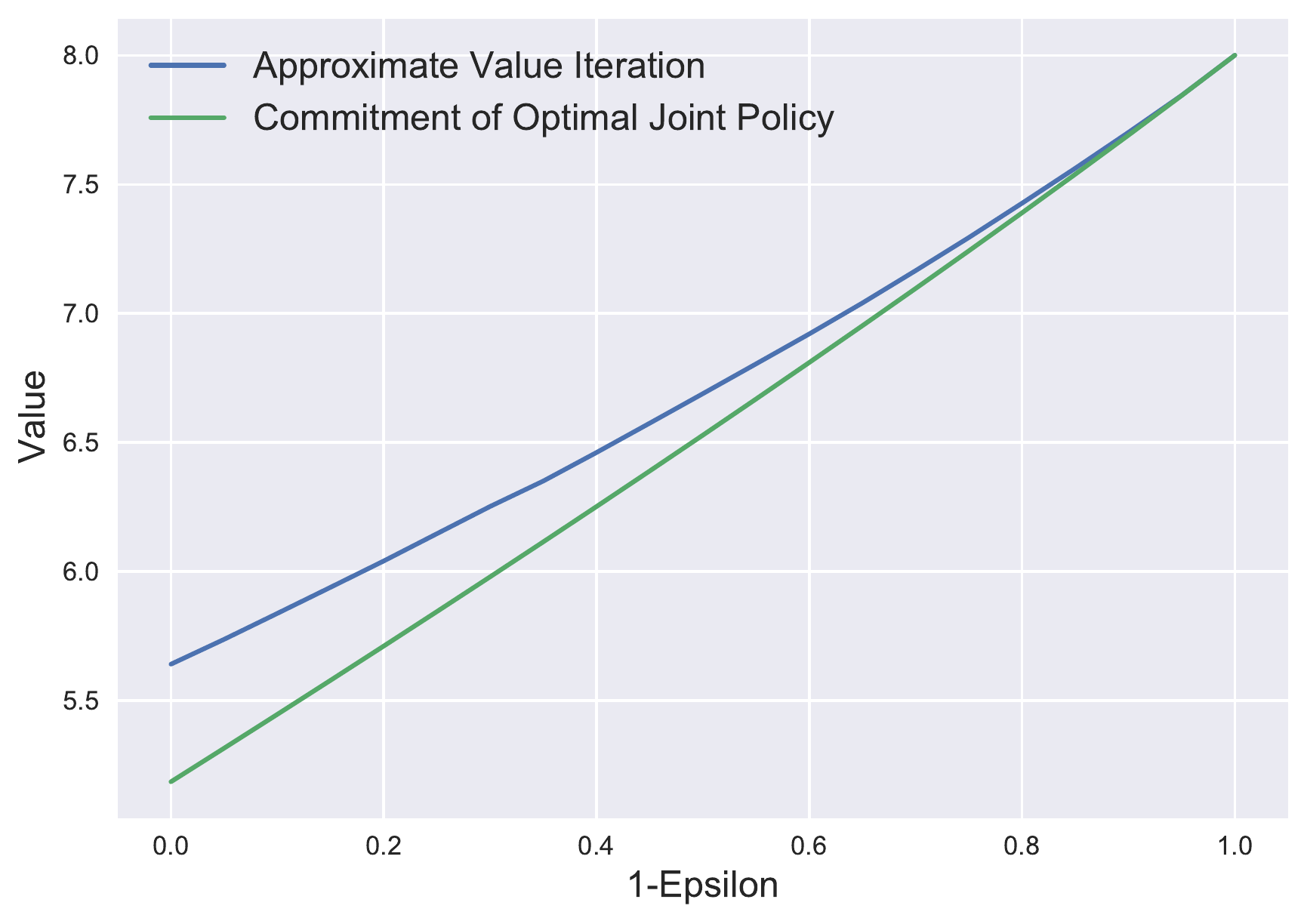}
    }
    \caption{Evaluation of Approximate Value Iteration for $\varepsilon$-Greedy Responses (Algorithm~\ref{algorithm:value_iteration_greedy}) in the Maze-Maker and Random MDP environment for decreasing values of $\varepsilon$. The green line describes the return of playing $\A_1$'s part of an optimal joint policy.}
    \label{figure:appendix_greedy_approx}
\end{figure}

\newpage 

\section{Experimental Details}\label{appendix:details_experiments}
The experiments were carried out on a virtual machine with $32$ CPUs, $60$GB RAM, and CentOS Linux 8 operating system. The experiments were implemented in Python 3.7 and the libraries matplotlib 3.2.1, numpy 1.20.1, and scipy 1.6.2 (for the linear program) were used. The code is available at \href{https://github.com/InteractiveIRL/src}{https://github.com/InteractiveIRL/src}. 

For the case of suboptimal responses and partial information, we assume that $\A_2$ responds with Boltzmann-rational policies with inverse temperature $\beta = 10$ in both environments. We assume that the inverse temperature, that is, the optimality of the second agent, is \emph{unknown} to the learner and must therefore be inferred. We simulate the partial information setting by generating trajectories according to policies $\pi^1_t$ and $\pi^2_t$ in episode $t$, where the length of the episode is random. More precisely, we let an episode end with probability $1-\gamma = 0.1$ each time step.\footnote{We impose a minimal trajectory length of $2$ time steps to prevent vacuous episodes.}

\subsection{Bayesian Interactive IRL}
We employ a Bayesian approach using the Metropolis-Hastings algorithm to sample from the posterior, with a uniform prior on the reward function and an exponential prior on the inverse temperature. Our approach is specified in Algorithm~\ref{algorithm:Bayesian_IRL_MCMC}.
As a proposal distribution for the reward function, we consider a discretisation of the $|\S|$-dimensional unit simplex $\Delta(\S)$ with step size $\delta$, similarly to \citep{ramachandran2007BIRL}. The Metropolis-Hastings algorithm then generates a Markov chain on the discretised simplex. To sample from the posterior given the last candidate $\r_{k-1}^t$ then means to choose a neighbour in the discretised simplex. This type of proposal distribution, which we refer to as Simplex Walk, proved to be a more efficient and robust sampling strategy as other proposal distributions (e.g.\ Dirichlet distributions). 
%We use a Dirichlet proposal distribution for the reward function and 
For the inverse temperature, we use a Gamma proposal distribution. % for the inverse temperature.
Similarly to Algorithm~\ref{algorithm:full_info_opt_behaviour}, we play greedily with respect to our current estimate of the true reward function. After sampling $K$ times from the posterior, we take the empirical means $\bar \r_t$ and $\bar \beta_t$ and compute an approximately optimal commitment strategy under $\bar \r_t$ and $\bar \beta_t$ by means of Algorithm~\ref{algorithm:value_iteration_boltzmann}. As a natural burn-in we use the last sampled reward and inverse temperature from episode $t$ as the first candidate in episode $t+1$. 

% In our experiments, we used the proposal distributions $g(\r \mid \r') \sim \text{Dir}(\alpha)$ and $g(\beta \mid \beta') \sim Gamma(\beta', 1+1/\beta')$. 
\begin{algorithm}[H]
\caption{Bayesian Interactive IRL via Simplex Walk}
\label{algorithm:Bayesian_IRL_MCMC}
\begin{algorithmic}[1]
\STATE \textbf{input:} $(\S, A_1, A_2, \trans, \gamma)$, priors $\P(\r)$, $\P(\beta)$, proposal distributions $g_1$, $g_2$, sample size~$K$
\STATE \textbf{initialise:} choose $\pi^1_1$ uniformly at random, sample $\r_0^0 \sim \P(\r)$ and $\beta_0^0 \sim \P(\beta)$
\FOR{$t = 1, 2, \dots $}
\STATE commit to policy $\pi^1_t$
\STATE observe trajectory $\tau_{t}$
\STATE // \texttt{sample from posterior via Metropolis-Hastings}
\FOR{$k=1, \dots, K$}
\STATE sample $\r \sim g_1 (\cdot \mid \r_{k-1}^t)$
\STATE sample $\beta \sim g_2(\cdot \mid \beta^t_{k-1})$
\STATE compute $p = \frac{\P ((\pi^1_1, \tau_1), \dots, (\pi^1_t, \tau_t) \mid \r, \beta) \P(\r) \P(\beta)} {g_1(\r \mid \r_{k-1}^t) g_2(\beta \mid \beta^t_{k-1})}$
\STATE \textbf{w.p.\ $\min\{ 1, \frac{p}{p_{k-1}}\}$:} $\r^t_k = \r$, $\beta^t_k = \beta$, $p^t_k = p$ 
\STATE \textbf{else:} $\r^t_k = \r_{k-1}$, $\beta^t_k = \beta_{k-1}$, $p^t_k = p^t_{k-1}$
\ENDFOR
\STATE set $\r^{t+1}_{0} = \r^t_K$, $\beta^{t+1}_0 = \beta^t_K$, $p^{t+1}_0 = p^t_K$
\STATE calculate mean reward function $\bar \r_t$ and beta $\bar \beta_t$
\STATE compute $\pi^1_{t+1}$ under $\bar \r_t$ and $\bar \beta_t$ via Algorithm~\ref{algorithm:value_iteration_boltzmann}
\ENDFOR
\end{algorithmic}
\end{algorithm}

\subsection{Environments: Maze-Maker}
In the Maze-Maker environment, agents $\A_1$ and $\A_2$ jointly control a cart in a $7\times7$ grid world. In this grid world, the doors leading from one cell to the neighbouring ones are locked. However, $\A_1$ can unlock exactly two doors at any time step before they fall shut again. 
$\A_2$ can attempt to move the cart through a door to a neighbouring cell. However, when the door is locked, the cart stays where it was. 
We assume that any attempted move of the cart succeeds with probability $0.8$ and that with probability $0.2$ the cart moves to a random neighbouring cell. 
Agents $\A_1$ and $\A_2$ are tasked with collecting three rewards of different value (+$1$, +$2$, +$3$), which are scattered in the grid world and disappear once collected. While $\A_2$ knows where the rewards are placed, $\A_1$ does not know their location. An illustration of the environment is given by Figure~\ref{figure:experiments_maze_maker}.
We model this environment as a two-agent MDP with $392$ states ($49 \times 8$) and discount factor $\gamma = 0.9$, where $\A_1$ has six actions (unlocking two out of four doors) and $\A_2$ four actions (attempting to move the cart North, East, South, West). 
As we consider a Stackelberg game, $\A_2$ knows beforehand which doors $\A_1$ will unlock. Therefore, $\A_1$ essentially selects a maze layout, which is communicated to $\A_2$ and through which $\A_2$ can move the cart.

\subsection{Details on Figure~\ref{figure:introductory_example}}\label{appendix:figure_1_explaination}
In Figure~\ref{figure:introductory_example}b, we assumed that $\A_2$ plays a Boltzmann-rational policy with inverse temperature $\beta = 10$. For simplicity and proper comparison, we assume that we can observe the fully specified Boltzmann policy played by $\A_2$ in each of the mazes. We use an adaption if Bayesian IRL \citep{ramachandran2007BIRL} and display the mean reward function in Figure~\ref{figure:introductory_example}b, where the colour scale, i.e.\ colour transparency, is obtained from the mean reward function in a given cell. More precisely, we use the Metropolis-Hastings algorithm with uniform prior and a Dirichlet proposal to sample from the posterior distribution $\P (\r \mid (\pi^1, \pi^2))$, where $\pi^1$ describes the maze layout.

% \section{Undiscounted Fixed-Horizon Episodic MDPs}\label{appendix:finite_horizon_MDPs}
% We chose to consider discounted infinite horizon MDPs as they are the standard setting used in the IRL literature.  
% Generally, the characterisation achieved in Theorem~\ref{lemma:feasible_set_characterization} and Corollary~\ref{corollary:feasible_set} can be translated to (undiscounted) fixed-horizon episodic MDPs. However, in undiscounted MDPs handling of the value function in connection to the reward function becomes more complicated. By artificially introducing a discount factor close to $1$, one can approximate constraints that are linear in the reward function and characterise the feasible set of reward functions. The proof of Theorem~\ref{proposition:ideal_environment} cannot be directly transferred to the undiscounted fixed-horizon setting because it relies on the \emph{precise} characterisation of feasible reward functions in Theorem~\ref{lemma:feasible_set_characterization}.
% The remainder of the results including the Bayesian formulation of the IRL problem can be easily translated to (undiscounted) fixed-horizon MDPs. In particular, the problem of approximating optimal commitment strategies, including Lemma~\ref{lemma:optimal_joint_policy_optimal}, Lemma~\ref{lemma:eps_greedy_dominating}, and Theorem~\ref{lemma:no_dominating_policy}, directly transfers to fixed horizons. The proposed approximate value iteration algorithms then inspire approximate backwards induction algorithms. 

\section{Influence}
Prior work on two-agent cooperation has considered measurements of how much one agent can influence the transition probabilities. \citet{dimitrakakis2017multi} define the influence of agent $\A_1$ (analogously for $\A_2$) on the transition probabilities as 
\begin{align*}
    \influence (\A_1)& = \max_s \max_{\acta_1, \acta_2, \actb} \lVert \trans(\cdot \mid s, \acta_1, \actb) - \trans(\cdot \mid s, \acta_2, \actb)\rVert_1, 
\end{align*}
which has also been adopted by \citet{radanovic2019learning} and \citet{ghosh2019towards}. They use this definition of influence to bound the performance gap when the beliefs or the behaviour of the two agents are misaligned. In our setting, however, the influence of an agent also relates to the IRL problem and our capacity to solve it. In particular, if $\influence ({\A_1}) = 0$, agent $\A_1$ does not influence the transition probabilities and it is therefore irrelevant what actions $\A_1$ takes. In terms of the IRL problem, we are then in the typical single-agent setting as $\A_2$ can ignore the presence of agent $\A_1$. %under the same transition kernel $\trans \equiv \trans_{\pi^1}$ for all policies $\pi^1$, i.e.\ independent of $\A_1$. 
On the other hand, if $\influence (\A_2) = 0$, then $\A_2$ does not influence transitions at all and the IRL problem becomes intractable as $\A_2$'s actions yield no information about the underlying reward function.

\end{document}